\documentclass{article}

\usepackage{hyperref}

 \usepackage[accepted]{icml2013}

\usepackage{times}
\usepackage{amsmath,amsthm,amsfonts,amssymb,mathrsfs}
\usepackage[fleqn,tbtags]{mathtools}
\usepackage{natbib}
\usepackage{graphicx}
\usepackage{latexsym}
\usepackage{epsfig}
\usepackage{epstopdf}

\icmltitlerunning{Sparse coding for multitask and transfer learning}

\input{tcilatex}
\begin{document} 

\twocolumn[
\icmltitle{Sparse coding for multitask and transfer learning}

% unless you've provided the [accepted] option to the icml2013
\icmlauthor{Andreas Maurer}{am@andreas-maurer.eu}
\icmladdress{Adalbertstrasse 55, D-80799, M�unchen, Germany}
\icmlauthor{Massimiliano Pontil}{m.pontil@cs.ucl.ac.uk}
\icmladdress{Department of Computer Science and Centre for Computational Statistics and Machine Learning\\
University College London, Malet Place, London WC1E 6BT, UK}
\icmlauthor{Bernardino Romera-Paredes}{bernardino.paredes.09@ucl.ac.uk}
\icmladdress{Department of Computer Science and UCL Interactive Centre\\
University College London, Malet Place, London WC1E 6BT, UK}

\icmlkeywords{boring formatting information, machine learning, ICML}

\vskip 0.18in
]

\begin{abstract}
We investigate the use of sparse coding and dictionary learning in the context of multitask and transfer learning. 
The central assumption of our learning method is that the tasks parameters are well approximated by sparse linear combinations of the atoms of a dictionary on a high or infinite dimensional space. This assumption, together with the large quantity of available data in the multitask and transfer learning settings, allows a principled choice of the dictionary. We provide bounds on the generalization error of this approach, for both settings. Numerical experiments on one synthetic and two real datasets show the advantage of our method over single task learning, a previous method based on orthogonal and dense representation of the tasks and a related method learning task grouping.
\iffalse
We present an extension of sparse coding to the problems of multitask and transfer learning. The central assumption of our learning method is that the task parameters are well approximated by sparse linear combinations of the atoms of a dictionary on a high or infinite dimensional space. This assumption, together with the large quantity of available data in the multitask and transfer learning settings, allows a principled choice of the dictionary. We provide bounds on the generalization error of this approach, for both settings. Numerical experiments on one synthetic and two real datasets show the advantage of our method over single task learning, a previous method based on orthogonal and dense representation of the tasks and a related method learning task grouping.
\fi
\end{abstract}

\section{Introduction}
The last decade has witnessed many efforts of the machine learning community to 
exploit assumptions of sparsity in the design of algorithms. A central development 
in this respect is the Lasso \cite{Tib}, which
estimates a linear predictor in a high dimensional space under a
regularizing $\ell _{1}$-penalty. Theoretical results guarantee a good
performance of this method under the assumption that the vector
corresponding to the underlying predictor is sparse, or at least has a small $\ell_{1}$-norm, see e.g. \cite{BvdG} and references therein. 

In this work we consider the case where the predictors are linear
combinations of the atoms of a dictionary of linear functions on a high or
infinite dimensional space, and we assume that we are free to choose
the dictionary. We will show that a principled choice is possible, if there
are many learning problems, or \textquotedblleft tasks", and there exists a
dictionary allowing sparse, or nearly sparse representations of all or most
of the underlying predictors. In such a case we can exploit the
larger quantity of available data to estimate the \textquotedblleft good" dictionary and
still reap the benefits of the Lasso for the individual tasks. 
This paper gives theoretical and experimental justification of this claim, both in the 
domain of multitask learning, where the new representation is applied to the 
tasks from which it was generated, and in the domain of learning to learn, 
where the dictionary is applied to new tasks of the same environment.

Our work combines ideas from sparse coding \cite{Ols2}, multitask learning \cite{Zhang,AEP,AMP,Bendavid,Caruana,EMP,Maurer2009} and 
learning to learn \cite{Baxter,thurn}.
There is a vast literature on these subjects and the list of papers provided here is necessarily incomplete.
%literature on MTL is vast
Learning to learn (also called inductive bias learning or transfer learning) has been
proposed by \citet{Baxter} and an error analysis is provided therein, showing that
a common representation which performs well on the training tasks will also generalize
to new tasks obtained from the same ``environment''.
%Along the same vein, an analysis for linear multitask feature learning is presented in \cite{Maurer2009}. 
The precursors of the analysis presented here are \cite{MP-ieee} and \cite{Maurer2009}.
The first paper provides a bound on the reconstruction error of sparse coding and may be seen
as a special case of the ideas presented here when the sample size is infinite.
The second paper provides a learning to learn analysis of the multitask feature learning method in \cite{AEP}.

We note that a method similar to the one presented in this paper has been recently proposed within the multitask learning setting \cite{Daume}. 
%Here we present a probabilistic analysis which complements well with the practical
% insights in [17], highlight the connection to sparse coding [25] and address the different % problem of learning to learn.
Here 
%The contributions of our paper are threefold: i) 
we highlight the connection between sparse coding and multitask learning and 
% which was not explored in the above work; ii) we 
present a probabilistic analysis which complements well with the practical insights in the above work. We also address the different problem of learning to learn, demonstrating the utility of our approach in this setting by means of both learning bounds and numerical experiments. A further novelty of our approach is that it applies to a Hilbert spaces setting, thereby providing the possibility of learning nonlinear predictors using reproducing kernel Hilbert spaces.
% in a natural way. However, we postpone the study of this setting to future work.

The paper is organized in the following manner. In Section \ref{sec:2}, we set up our notation and introduce the learning problem. In Section \ref{sec:bounds}, we present our learning bounds for multitask learning and learning to learn. In Section \ref{sec:exp} we report on numerical experiments. Section \ref{sec:discussion} contains concluding remarks.

%We note that at the time of the paper writing a method very similar to ours has been proposed for multitask learning \cite{Daume}. Here we present a probabilistic analysis which complements well with the practical insights in \cite{Daume}, highlight the connection to sparse coding \cite{Ols1} and address the different problem of learning to learn.

\section{Method}
\label{sec:2}
In this section, we turn to a technical exposition of the
proposed method, introducing some necessary notation on the way.

Let $H$ be a finite or infinite dimensional Hilbert space with inner
product $\left\langle \cdot,\cdot\right\rangle $, norm $\left\Vert \cdot \right\Vert $,
and fix an integer $K$. We study the problem 
\begin{equation}
\min_{D\in \mathcal{D}_{K}}\frac{1}{T}\sum_{t=1}^{T}\min_{\gamma \in 
\mathcal{C}_{\alpha }}\frac{1}{m}\sum_{i=1}^{m}\ell \left( \left\langle
D\gamma ,x_{ti}\right\rangle ,y_{ti}\right) ,  \label{basic algorithm}
\end{equation}%
where

%MASSI: possibly organize remove the bullets - have the definition before the method
%plus have a short informal story of what the method is about: both in multitask learning and and learning to learn setting we are provided with multiple training set, each corresponding to a task and we seek a a good representation for these tasks. Of course, the way the representation is evaluated differ and we now described in detail thee two settings
\begin{itemize}
\item $\mathcal{D}_{K}$ is the set of $K$-dimensional dictionaries (or simply
dictionaries), which means that every $D\in \mathcal{D}_{K}$ is a linear map 
$D:%
%TCIMACRO{\U{211d} }%
%BeginExpansion
\mathbb{R}
%EndExpansion
^{K}\rightarrow H$, such that $\left\Vert De_{k}\right\Vert \leq 1$ for
every one of the canonical basis vectors $e_{k}$ of $%
%TCIMACRO{\U{211d} }%
%BeginExpansion
\mathbb{R}
%EndExpansion
^{K}$. The number $K$ can be regarded as one of the regularization
parameters of our method.

\item $\mathcal{C}_{\alpha }$ is the set of code vectors $\gamma $ in $%
%TCIMACRO{\U{211d} }%
%BeginExpansion
\mathbb{R}
%EndExpansion
^{K}$ satisfying $\left\Vert \gamma \right\Vert _{1}\leq \alpha $. The
$\ell_1$-norm constraint implements the assumption of sparsity and $\alpha $ is the
other regularization parameter. Different sets ${\cal C}_\alpha$ could be readily used in our method, such as those associated with $\ell_p$-norms.%
%or mixed-norm, see e.g. \cite{Jenatton}.

\item $\mathbf{Z}=\left( \left( x_{ti},y_{ti}\right) :1\leq i\leq m,1\leq
t\leq T\right) $ is a dataset on which our algorithm operates. Each $x_{ti}\in H$
represents an input vector, and $y_{ti}$ is a corresponding
real valued label. We also write $\mathbf{Z}=\left( \mathbf{X},\mathbf{Y}%
\right) =\left( \mathbf{z}_{1},\ldots,\mathbf{z}_{T}\right) =\left( \left( 
\mathbf{x}_{1},\mathbf{y}_{1}\right) ,\ldots,\left( \mathbf{x}_{T},\mathbf{y}%
_{T}\right) \right) $ with $\mathbf{x}_{t}=\left( x_{t1},\ldots,x_{tm}\right) $
and $\mathbf{y}_{t}=\left( y_{t1},\ldots,y_{tm}\right) $. 
The index $t$ identifies a learning task, and $\mathbf{z}_{t}$ are the
corresponding training points, 
%The index $t$ identifies a learning task, and $\mathbf{z}_{t}$ are the
%MMM corresponding training points, PPP
% data belonging to the $t$-th task, 
so the algorithm operates on $T$ tasks,
each of which is represented by $m$ example pairs.

\item $\ell $ is a loss function where $\ell \left( y,y^{\prime }\right) $
measures the loss incurred by predicting $y$ when the true label is $%
y^{\prime }$. We assume that $\ell $ has values in $\left[ 0,1\right] $ and
has Lipschitz constant $L$ in the first argument for all values of the
second argument.
\end{itemize}

The minimum in (\ref{basic algorithm}) is zero if the data is generated
according to a noise-less model which postulates that there is a  ``true''
dictionary $D^{\ast }\in \mathcal{D}_{K^{\ast }}$ with $K^{\ast }$ atoms and
vectors $\gamma _{1}^{\ast },\ldots,\gamma _{T}^{\ast }$ satisfying $\left\Vert
\gamma _{t}^{\ast }\right\Vert _{1}\leq \alpha ^{\ast }$, such that an input 
$x\in H$ generates the label $y=\left\langle D^{\ast }\gamma _{t}^{\ast
},x\right\rangle $ in the context of task $t$. 
If $K\geq K^{\ast }$ and $%
\alpha \geq \alpha ^{\ast }$ then the minimum in (\ref{basic algorithm})
is zero. In Section \ref{sec:exp}, we will present experiments with such a generative model,
when noise is added to the labels, that is $y=\left\langle D^{\ast }\gamma
_{t}^{\ast },x\right\rangle +\zeta $ with $\zeta \sim {\mathcal N}\left(
0,\sigma \right)$, the standard normal distribution.

The method (\ref{basic algorithm}) should output a minimizing $D\left( \mathbf{Z}\right) \in 
\mathcal{D}_{K}$ as well as a minimizing $\gamma _{1}\left( \mathbf{Z}%
\right) ,\ldots,\gamma _{T}\left( \mathbf{Z}\right) $ corresponding to the
different tasks. Our implementation, described in Section \ref{sec:OA}, does not guarantee
exact minimization, because of the non-convexity of the problem. 
Below predictors are always linear, specified by a vector $w\in H$, 
predicting the label $\left\langle w,x\right\rangle $ for an input $x\in H
$, and a learning algorithm is a rule which assigns a predictor $A\left( 
\mathbf{z}\right) $ to a given data set $\mathbf{z}=\left( \left(
x_{i},y_{i}\right) :1\leq i\leq m\right) \in \left( H\times 
%TCIMACRO{\U{211d} }%
%BeginExpansion
\mathbb{R}
%EndExpansion
\right) ^{m}$.

\iffalse
%MASSI
We remark that a method similar to \eqref{basic algorithm} has recently been proposed by \citet{Daume}. There, the Frobenius norm on the dictionary $D$ is used in place of the $\ell_2/\ell_\infty$-norm employed here and the $\ell_1/ \ell_1$ norm on the coefficient matrix $[\gamma_1,\dots,\gamma_T]$ is used in place of the $\ell_1/ \ell_\infty$. Furthermore, our method depends on one regularization parameter (the upper bound $\alpha$) whereas their method uses two regularisation parameters, denoted by $\lambda$ and $\mu$ in \citep[eq.~(1)]{Daume}. However one of these parameters is redundant and may be eliminated without restricting the set of solutions obtained by their method. 
\fi

%Further variations or algorithm \eqref{basic algorithm} can be obtained by changing the sets ${\cal D}_K$ and ${\cal C}_\alpha$.

%We now describe the applications of method \eqref{basic algorithm} to multitask learning and learning to learn in turn. 
%These problems different in the way the multisample ${\mathbf Z}$ is generated and how the learned dictionary is evaluated. 
\section{Learning bounds}
\label{sec:bounds}

In this section, we present learning bounds for method \eqref{basic algorithm}, both in the multitask learning and learning to learn settings, and discuss the special case of sparse coding.
\subsection{Multitask learning}

Let $\mu _{1},\ldots,\mu _{T}$ be probability measures on $H\times 
%TCIMACRO{\U{211d} }%
%BeginExpansion
\mathbb{R}
%EndExpansion
$. We interpret $\mu _{t}\left( x,y\right) $ as the probability of observing
the input/output pair $\left( x,y\right) $ in the context of task $t$. For
each of these tasks an i.i.d. training sample $\mathbf{z}_{t}=\left( \left(
x_{ti},y_{ti}\right) :1\leq i\leq m\right) $ is drawn from $\left( \mu _{t}\right) ^{m}$ and the ensemble $\mathbf{Z}\sim
\prod_{t=1}^{T}\mu _{t}^{m}$ is input to algorithm \eqref{basic algorithm}. Upon returning of a minimizing $D\left( \mathbf{Z}\right) $ and $\gamma _{1}\left( \mathbf{Z}%
\right) ,\ldots,\gamma _{T}\left( \mathbf{Z}\right) $, we will use the
predictor $D\left( \mathbf{Z}\right) \gamma _{t}\left( \mathbf{Z}\right) $
on the $t$-th task. The average over all tasks of the expected error
incurred by these predictors is%
\begin{equation*}
\frac{1}{T}\sum_{t=1}^{T}{\mathbb E}_{\left( x,y\right) \sim \mu _{t}}\left[ \ell
\left( \left\langle D\left( \mathbf{Z}\right) \gamma _{t}\left( \mathbf{Z}%
\right) ,x\right\rangle ,y\right) \right] .
\end{equation*}%
We compare this \textit{task-average risk} to the minimal analogous risk
obtainable by any dictionary $D \in {\cal D}_K$ and any set of vectors $\gamma
_{1},\dots,\gamma _{T} \in \mathcal{C}_{\alpha }$. Our first result is a bound
on the excess risk. 
\begin{theorem}
\label{Theorem Multitask}Let $\delta >0$ and let $\mu _{1},\ldots,\mu _{T}$ be
probability measures on $H\times \mathbb{R}$. With probability at least $1-\delta $ in the draw of $\mathbf{Z}\sim
\prod_{t=1}^{T}\mu _{t}^{m}$ we have 
\begin{multline*}
\frac{1}{T}\sum_{t=1}^{T}{\mathbb E}_{\left( x,y\right) \sim \mu _{t}}\left[ \ell
\left( \left\langle D\left( \mathbf{Z}\right) \gamma _{t}\left( \mathbf{Z}%
\right) ,x\right\rangle ,y\right) \right] \\
-\inf_{D\in \mathcal{D}_{K}}
\frac{1}{T} \sum_{t=1}^{T} \inf_{\gamma \in\mathcal{C}_{\alpha }}
{\mathbb E}_{\left( x,y\right) \sim \mu _{t}}\left[ \ell \left(
\left\langle D\gamma,x\right\rangle ,y\right) \right]  \\
\leq L\alpha \sqrt{\frac{2S_{1}\left( \mathbf{X}\right) \left( K+12\right) }{%
mT}} \\
+L\alpha \sqrt{\frac{8S_{\infty }\left( \mathbf{X}\right) \ln \left(
2K\right) }{m}} 
+\sqrt{\frac{8\ln 4/\delta }{mT}},
\end{multline*}%
where $S_{1}\left( \mathbf{X}\right) =
\frac{1}{T}
\sum_{t=1}^{T}{\rm tr}\left( \hat{\Sigma}\left( \mathbf{x}_{t}\right) \right) $
and $S_{\infty }\left( \mathbf{X}\right) =\frac{1}{T}
\sum_{t=1}^{T}\lambda _{\max }\left( \hat{\Sigma}\left( \mathbf{x}%
_{t}\right) \right) $. Here $\hat{\Sigma}\left( \mathbf{x}_{t}\right) $ is
the empirical covariance of the input data for the $t$-th task, ${\rm tr}\left(
\cdot\right) $ denotes the trace and $\lambda _{\max }(\cdot)$ the largest
eigenvalue.
\end{theorem}
We state several implications of this theorem.

%% MASSI: one line summary in itaic for each???
\begin{enumerate}
\item The quantity $S_{1}\left( \mathbf{X}\right) $ appearing in the bound 
is just the average square norm of the input data points, while $S_{\infty
}\left( \mathbf{X}\right) $ is roughly the average inverse of the observed
dimension of the data for each task. Suppose that $H=%
%TCIMACRO{\U{211d} }%
%BeginExpansion
\mathbb{R}
%EndExpansion
^{d}$ and that the data-distribution is uniform on the surface of the unit
ball. Then $S_{1}\left( \mathbf{X}\right) =1$ and for $m\ll d$ it follows
from Levy's isoperimetric inequality (see e.g. \cite{Ledoux1991}) that $S_{\infty }\left( \mathbf{X}%
\right) \approx 1/m$, so the corresponding term behaves like $\sqrt{\ln K}/m$%
. If the minimum in (\ref{basic algorithm}) is small and $T$ is large enough for this term
to become dominant then there is a significant advantage of the method over learning the tasks independently. If
the data is essentially low dimensional, then $S_{\infty }\left( \mathbf{X}%
\right) $ will be large, and in the extreme case, if the data is one-dimensional
for all tasks then $S_{\infty }\left( \mathbf{X}\right) =S_{1}\left( \mathbf{%
X}\right) $ and our bound will always be worse by a factor of $\ln K$ than
standard bounds for independent single task learning as in \cite{Bartlett2002}. This makes sense, because for low dimensional data there can be
little advantage to multitask learning.

\item In the regime $T<K$ the bound is dominated by the term of
order $\sqrt{S_{1}\left( \mathbf{X}\right) K/mT}>\sqrt{S_{1}\left( \mathbf{X}
\right) /m}$. This is easy to understand, because the dictionary atoms $De_{k}$ 
can be chosen independently, separately for each task, so we could
at best recover the usual bound for linear models and there is no benefit
from multitask learning.

\item Consider the noiseless generative model mentioned in Section \ref{sec:2}.
If $K\geq K^{\ast }$ and $\alpha \geq \alpha ^{\ast }$ then the minimum in \eqref{basic algorithm} is zero. In the bound the overestimation of $K^{\ast }$
can be compensated by a proportional increase in the number of tasks
considered and an only very minor increase of the sample size $m$, namely $%
m\rightarrow \left( \ln K^{\ast }/\ln K\right) m$.

\item Suppose that we concatenate two sets of tasks. If the tasks are
generated by the model described in Section \ref{sec:2} then the resulting set of
tasks is also generated by such a model, obtained by concatenating the lists
of atoms of the two true dictionaries $D_{1}^{\ast }$ and $D_{2}^{\ast }$ to
obtain the new dictionary $D^{\ast }$ of length $K^{\ast }=K_{1}^{\ast
}+K_{2}^{\ast }$ and taking the union of the set of generating vectors $%
\left\{ \gamma _{t}^{\ast 1}\right\} _{t=1}^{T}$ and $\left\{ \gamma
_{t}^{\ast 2}\right\} _{t=1}^{T}$, extending them to $%
%TCIMACRO{\U{211d} }%
%BeginExpansion
\mathbb{R}
%EndExpansion
^{K_{1}^{\ast }+K_{2}^{\ast }}$ so that the supports of the first group are
disjoint from the supports of the second group. If $T_{1}=T_{2}$, $%
K_{1}^{\ast }=K_{2}^{\ast }$ and we train with the correct parameters, then
the excess risk for the total task set increases only by the order of $1/%
\sqrt{m}$, independent of $K$, despite the fact that the tasks in the second
group are in no way related to those in the first group. Our method has the
property of finding the right clusters of mutually related tasks.

\item Consider the alternative method of subspace learning (SL) where $%
\mathcal{C}_{\alpha }$ is replaced by an euclidean ball of radius $\alpha$.
With similar methods one can prove a bound for SL where, apart from slightly
different constants, $\sqrt{\ln K}$ above is replaced by $K$. SL will be successful
and outperform the proposed method, whenever $K$ can be chosen small, with $%
K<m$ and the vector $\gamma _{t}^{\ast }$ utilize the entire span of the
dictionary. For large values of $K$, a correspondingly large number of tasks
and sparse $\gamma _{t}^{\ast }$ the proposed method will be superior.
\end{enumerate}

The proof of Theorem \ref{Theorem Multitask}, which is given in Section \ref{sec:AMTL} of the supplementary appendix, uses standard
methods of empirical process theory, but also employs a concentration result
related to Talagrand's convex distance inequality to obtain the crucial
dependence on $S_{\infty }\left( \mathbf{X}\right) $. At the end of Section \ref{sec:AMTL}
we sketch applications of the proof method to other regularization schemes,
such as the one presented in \cite{Daume}, in which the Frobenius norm on the dictionary $D$ is used in place of the $\ell_2/\ell_\infty$-norm employed here and the $\ell_1/ \ell_1$ norm on the coefficient matrix $[\gamma_1,\dots,\gamma_T]$ is used in place of the $\ell_1/ \ell_\infty$. 
\subsection{Learning to learn}

There is no absolute way to assess the quality of a learning algorithm.
Algorithms may perform well on one kind of task, but poorly on another kind.
It is important that an algorithm performs well on those tasks which it is
likely to be applied to. To formalize this, \citet{Baxter}
introduced the notion of an \textit{environment}, which is a probability
measure $\mathcal{E}$ on the set of tasks. Thus $\mathcal{E}\left( \tau
\right) $ is the probability of encountering the task $\tau $ in the
environment $\mathcal{E}$, and $\mu _{\tau }\left( x,y\right) $ is the
probability of finding the pair $\left( x,y\right) $ in the context of the
task $\tau $.

Given $\mathcal{E}$, the {\em transfer risk} (or simply risk) of a learning algorithm $A$ is defined as
follows. We draw a task from the environment, $\tau \sim \mathcal{E}$, which
fixes a corresponding distribution $\mu _{\tau }$ on $H\times 
%TCIMACRO{\U{211d} }%
%BeginExpansion
\mathbb{R}
%EndExpansion
$. Then we draw a training sample $\mathbf{z}\sim \mu _{\tau }^{m}$ and use
the algorithm to compute the predictor $A\left( \mathbf{z}\right) $. Finally
we measure the performance of this predictor on test points $\left(
x,y\right) \sim \mu _{\tau }$. The corresponding definition of the transfer risk of $A
$ reads as%
\begin{equation}
%\hspace{-.1truecm}
R_{\mathcal{E}}\left( A\right) ={\mathbb E}_{\tau \sim \mathcal{E}}{\mathbb E}_{\mathbf{z}\sim
\mu _{\tau }^{m}}{\mathbb E}_{\left( x,y\right) \sim \mu _{\tau }}\left[ \ell \left(
\left\langle A\left( \mathbf{z}\right) ,x\right\rangle ,y\right) \right]
\label{finition of transfer risk}
\end{equation}%
which is simply the expected loss incurred by the use of the algorithm $A$ on
tasks drawn from the environment $\mathcal{E}$. 

For any given dictionary $D\in \mathcal{D}_{K}$ we consider the learning
algorithm $A_{D}$, which for $\mathbf{z}\in \mathcal{Z}^{m}$ computes the
predictor%
\begin{equation}
A_{D}\left( \mathbf{z}\right) =D~\arg \min_{\gamma \in \mathcal{C}_\alpha}\frac{1}{m%
}\sum_{i=1}^{m}\ell \left( \left\langle D\gamma ,x_{i}\right\rangle
,y_{i}\right) .  \label{Algorithm AD}
\end{equation}%
Equivalently, we can regard $A_{D}$ as the Lasso operating on data
preprocessed by the linear map $D^\top$, the adjoint of $D$.

We can make a single observation of the environment $\mathcal{E}$ in the following way:
one first draws a task $\tau \sim \mathcal{E}$. This
task and the corresponding distribution $\mu _{\tau }$ are then
observed by drawing an i.i.d. sample $\mathbf{z}$ from $\mu _{\tau }$, that is $%
\mathbf{z}\sim \mu _{\tau }^{m}$. For simplicity the sample size $m$ will be
fixed. Such an observation corresponds to the draw of a sample $\mathbf{z}$
from a probability distribution $\rho _{\mathcal{E}}$ on $\left( H\times 
%TCIMACRO{\U{211d} }%
%BeginExpansion
\mathbb{R}
%EndExpansion
\right) ^{m}$ which is defined by 
\begin{equation}
\rho _{\mathcal{E}}\left( \mathbf{z}\right) :={\mathbb E}_{\tau \sim \mathcal{E}}\left[
\left( \mu _{\tau }\right) ^{m}\left( \mathbf{z}\right) \right] .
\label{Definition RHO_EPS}
\end{equation}%
To estimate an environment a large number $T$ of independent observations is
needed, corresponding to a vector $\mathbf{Z}=\left( \mathbf{z}_{1},\ldots,%
\mathbf{z}_{T}\right) \in \left( \left( H\times 
%TCIMACRO{\U{211d} }%
%BeginExpansion
\mathbb{R}
%EndExpansion
\right) ^{m}\right) ^{T}$ drawn i.i.d. from $\rho _{\mathcal{E}}$, that is $%
\mathbf{Z}\sim \left( \rho _{\mathcal{E}}\right) ^{T}$.

We now propose to solve the problem (\ref{basic algorithm}) with the data $\mathbf{Z}$, ignore the resulting 
$\gamma _{i}\left( \mathbf{Z}\right) $, but retain the dictionary $D\left( 
\mathbf{Z}\right) $ and use the algorithm $A_{D\left( \mathbf{Z}%
\right) }$ on future tasks drawn from the same environment. The performance of this method can be quantified as the transfer risk $R_{%
\mathcal{E}}\left( A_{D\left( \mathbf{Z}\right) }\right) $ as defined in equation (\ref{finition of transfer risk}) and again we are interested in
comparing this to the risk of an ideal solution based on complete knowledge
of the environment. For any fixed dictionary $D$ and task $\tau $ the best
we can do is to choose $\gamma \in \mathcal{C}$ so as to minimize ${\mathbb E}_{\left(
x,y\right) \sim \mu _{\tau }}\left[ \ell \left( \left\langle D\gamma
,x\right\rangle ,y\right) \right] $, so the best is to choose $D$ so as to
minimize the average of this over $\tau \sim \mathcal{E}$. The quantity%
\begin{equation*}
R_{\rm opt}=\min_{D\in \mathcal{D}_K}{\mathbb E}_{\tau \sim \mathcal{E}}\min_{\gamma \in 
\mathcal{C}_\alpha}{\mathbb E}_{\left( x,y\right) \sim \mu _{\tau }}\ell \left[ \left(
\left\langle D\gamma ,x\right\rangle ,y\right) \right] 
\end{equation*}%
thus describes the optimal performance achievable under the given
constraint. Our second result is

\begin{theorem}
\label{Theorem Main}With probability at least $1-\delta $ in the
multisample $\mathbf{Z}=\left(\mathbf{X},\mathbf{Y}\right)
%\left( \mathbf{z}_{1},\ldots,\mathbf{z}_{T}\right)
\sim \rho _{\mathcal{E}}^{T}$ we have 
$$
R_{\mathcal{E}}\left( A_{D\left( \mathbf{Z}\right) }\right) - R_{\rm opt} \leq 
L\alpha K\sqrt{\frac{2\pi S_{1}\left( \mathbf{X}\right) }{T}}$$
$$+4L\alpha 
\sqrt{\frac{S_{\infty }\left( \mathcal{E}\right) \left( 2+\ln K\right) }{m}}+%
\sqrt{\frac{8\ln 4/\delta }{T}},
$$
where $S_{1}\left( \mathbf{X}\right) $ is as in Theorem \ref{Theorem
Multitask} and $S_{\infty }\left( \mathcal{E}\right) :={\mathbb E}_{\tau \sim \mathcal{E}}{\mathbb E}_{\left( 
\mathbf{x,y}\right) \sim \mu _{\tau }^{m}}\lambda _{\max }\left( \hat{\Sigma}%
\left( \mathbf{x}\right) \right)$.
\end{theorem}

We discuss some implications of the above theorem.
%, some of which are analogous to the remarks following Theorem 
\ref{Theorem Multitask}. 

\begin{enumerate}
\item The interpretation of $S_{\infty }\left( \mathcal{E}\right) $ is
analogous to that of $S_{\infty }\left( \mathbf{X}\right) $ in the bound for Theorem \ref{Theorem Multitask}. The same applies to Remark 6 following Theorem \ref{Theorem Multitask}. 
\item In the regime $T\leq K^{2}$ the result does not imply any useful
behaviour. On the other and, if $T\gg K^{2}$ the dominant term in the bound 
%estimation error 
is of order $\sqrt{S_{\infty }\left( \mathcal{E}\right) /m}$.
%MASSI: I COMMENTED THIS:, the estimation error becoming independent of $K$ in the limit $T\rightarrow \infty $.

\item There is an important difference with the multitask learning bound, namely 
in Theorem \ref{Theorem Main}  we have $\sqrt{T}$ in the denominator of the first term of the excess
risk, and not $\sqrt{mT}$ as in Theorem \ref{Theorem Multitask}. This is because in the
setting of learning to learn there is always a possibility of being misled
by the draw of the training tasks. This possibility can only decrease as $T$ increases -- increasing $m$ does not help.

\end{enumerate}

The proof of Theorem \ref{Theorem Main} is given in Section \ref{sec:ALTL} of the supplementary appendix and follows the method outlined in \cite{Maurer2009}: one first bounds the estimation error for the expected empirical risk on future tasks, and then combines this with a
bound of the expected true risk by said expected empirical risk. The term $K/\sqrt{T}$ may be an artefact of our method of proof and the
conjecture that it can be replaced by $\sqrt{K/T}$ seems plausible. 

\iffalse
{\bf Example.} We discuss an example of the problem in which the data is generated by a linear regression model as $y_{ti} = \lb w_t,x_{ti} + \xi_{ti}\rangle$, where the task vectors are sampled from a measure $\rho$ on $B_1(H)$, the inputs $x_{ti}$ are sample i.i.d.. from the uniform measure on $B_1(H)$, denote by $p$, and $\xi_{ti}$ are i.i.d. Gaussian noise variables 
with standard deviation $\sigma$. In other words, this model is equivalent to choose $\mu(x,y) = p(x) \exp(-(y-\lb w,x\rangle)/2\sigma^2)$. 
Note that we do not assume that these tasks are obtained as a sparse combination of some dictionary.
Furthermore we consider the square loss. 
Under these assumptions we have that $\{\mathbb E}_{(x,y) \sim\mu_t} \ell(\lb w,x\rangle,y) = \|w_t-w\|^2$.
\fi

\subsection{Connection to sparse coding}\label{sec:Sparse Coding}
We discuss a special case of Theorem \ref{Theorem Main} in the limit $m \rightarrow \infty$, showing that it subsumes the sparse coding result in \cite{MP-ieee}. To this end, we assume the noiseless generative model $y_{ti} = \langle w_t,x_{ti} \rangle$ described in Section \ref{sec:2}, 
that is $\mu(x,y) = p(x) \delta(y,\langle w,x\rangle)$, where $p$ is the uniform distribution on the sphere in ${\mathbb R}^d$ (i.e. the Haar measure). In this case the environment of tasks is fully specified by a measure $\rho$ on the unit ball in ${\mathbb R}^d$ from which a task $w \in {\mathbb R}^d$ is drawn and 
the measure $\mu$ is identified with the vector $w$.
Note that we do not assume that these tasks are obtained as sparse combinations of some dictionary.
Under the above assumptions and choosing $\ell$ to be the square loss, we have that ${\mathbb E}_{(x,y) \sim\mu_t} \ell(\langle w,x\rangle,y) = \|w_t-w\|^2$. Consequently, in the limit of $m \rightarrow \infty$ method \eqref{basic algorithm} reduces to a constrained version of sparse coding \cite{Ols2}, namely
$$
\min_{D \in {\cal D}_K} \frac{1}{T} \sum_{t=1}^T \min_{\gamma \in {\cal C}_\alpha} \|D\gamma - w_t\|^2.
$$
In turn, the transfer error of a dictionary $D$ is given by the quantity $%
R(D):=\min_{\gamma \in \mathcal{C}_{\alpha }}\Vert D\gamma -w\Vert ^{2}$ and
$R_{\mathrm{opt}}=\min_{D\in \mathcal{D}_{K}}{\mathbb{E}}_{w\sim \rho
}\min_{\gamma \in \mathcal{C}_{\alpha }}\Vert D\gamma -w\Vert ^{2}$. Given
the constraints $D\in \mathcal{D}_{K}$, $\gamma \in \mathcal{C}_{\alpha }$
and $\left\Vert x\right\Vert \leq 1$, the square loss $\ell \left(
y,y^{\prime }\right) =\left( y-y^{\prime }\right) ^{2}$, evaluated at $%
y=\left\langle D\gamma ,x\right\rangle $, can be restricted to the interval $%
y\in \left[ -\alpha ,\alpha \right] $, where it has the Lipschitz constant $%
2\left( 1+\alpha \right) $ for any $y^{\prime }\in \left[ -1,1\right] $, as
is easily verified. Since $S_{1}(\mathbf{X})=1$ and $S_{\infty }\left( 
\mathcal{E}\right) <\infty $, the bound in Theorem \ref{Theorem Main}
becomes 
\begin{equation}
R(D)-R_{\mathrm{opt}}\leq 2\alpha (1+\alpha )K\sqrt{\frac{2\pi }{T}}+8\sqrt{%
\frac{\ln 4/\delta }{T}}  \label{eq:goggo}
\end{equation}%
in the limit $m\rightarrow \infty $. The typical choice for $\alpha$ is $\alpha \leq 1$, which ensures that $\|D\gamma\| \leq 1$. In this case inequality \eqref{eq:goggo} provides an improvement over the sparse coding bound in \cite{MP-ieee} (cf. Theorem 2 and Section 2.4 therein), which contains an additional term of the order of $\sqrt{(\ln T)/T}$ and the same leading term in $K$ as in \eqref{eq:goggo} but with slightly worse constant ($14$ instead of $4 \sqrt{2\pi}$). The connection of our method to sparse
coding is experimentally demonstrated in Section \ref{sec:Real experiments} and illustrated in Figure \ref{fig:dictionaries}.

\section{Experiments}
\label{sec:exp}

In this section, we present experiments on a synthetic and two real datasets. 
The aim of the experiments is to study the statistical performance of the proposed method, in both settings of multitask learning and learning to learn. We compare our method, denoted as Sparse Coding Multi Task Learning (SC-MTL), with independent ridge regression (RR) as a base line and multitask feature learning (MTFL) \cite{AEP} and GO-MTL \cite{Daume}.  We also report on sensitivity analysis of the proposed method versus different number of parameters involved. 

%The aim of the expedients is twofold. First, wish to study the statistical performance of the proposed method, both in a multitask and learning to learning setting. We compare the method with independent ridge regression as a base line and multitask feature learning \cite{AEP}.  Second, we study the computational properties and sensitivity analysis of the proposed method versus the number of paramters. Here we show that despite the nonconvexity of the problem, the method finds a good solution etc etc.

%should we compare to other methods?
\subsection{Optimization algorithm}
\label{sec:OA}
We solve problem \eqref{basic algorithm} by alternating minimization over
the dictionary matrix $D$ and the code vectors ${\mathbf{\gamma }}$. The
techniques we use are very similar to standard methods for sparse coding and
dictionary learning, see e.g. \cite{Jenatton} and references therein for more
information. Briefly, assuming that the loss function $\ell $ is convex
and has Lipschitz continuous gradient, either minimization problem is
convex and can be solved efficiently by proximal gradient methods, see e.g. \cite%
{BT,combettes}. The key ingredient in each step is the computation of the
proximity operator, which in either problem has a closed form expression. %
\iffalse Specifically, the optimization over $D$ involves the proximity
operator given by the projection onto the $\ell _{2}$ ball. The
minimization over the vectors $\gamma _{t}$ can be solved one task a
a time, and the proximity operator of the function $\delta _{\{\Vert
u}\Vert _{2}\leq 1\}}$ is given by $\mathrm{prox}(u)=u$ if$\Vert
u\Vert \leq 1$ and $u/\Vert u\Vert $ otherwise. The
minimization over $\gamma $ requires solving many constrained Lasso
problems. Here we need to know the proximity operator of the function $%
\delta _{\{\Vert c\Vert _{1}\leq \alpha \}}$, see e.g. \cite{duchi}. %
\fi

\subsection{Toy experiment}

\begin{figure}[t]
\begin{center}
\includegraphics[width=0.38\textwidth]{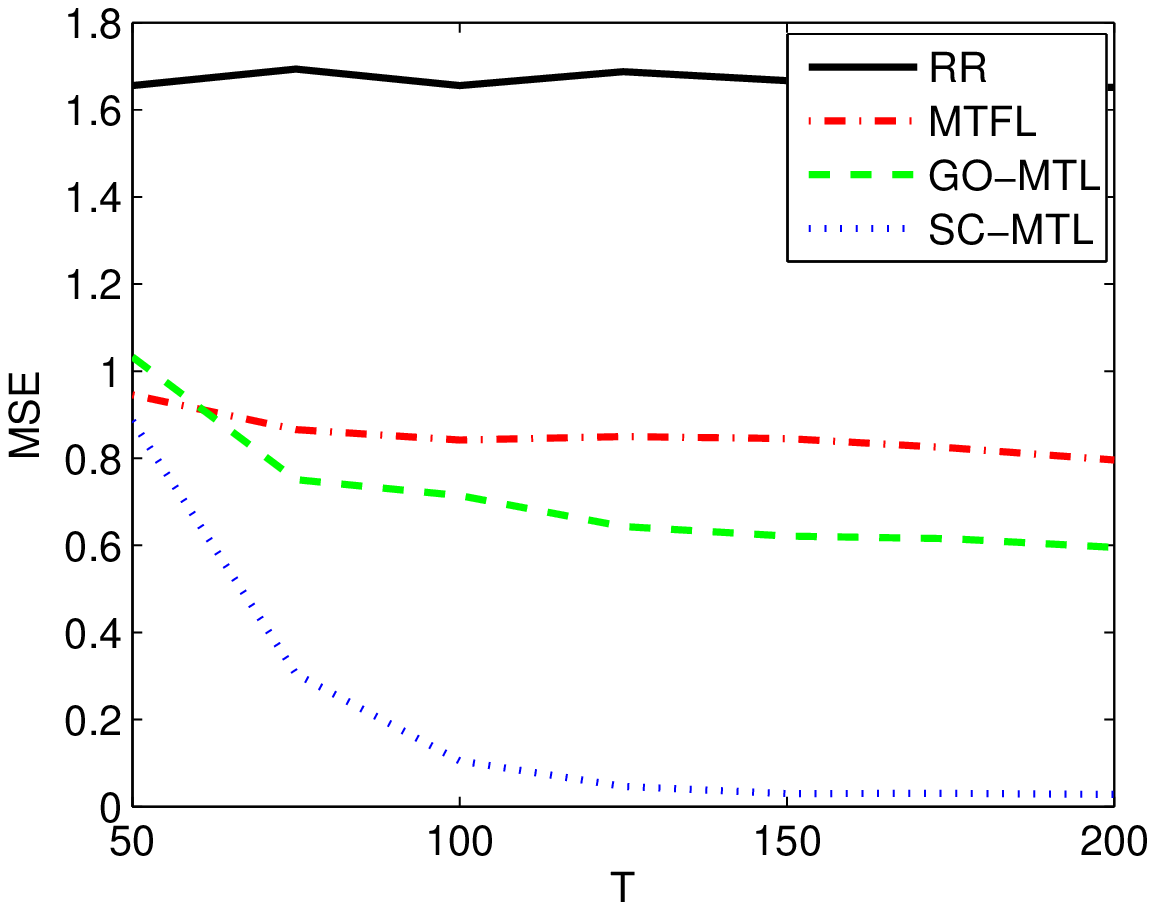}
\vspace{-.48truecm}
\includegraphics[width=0.38\textwidth]{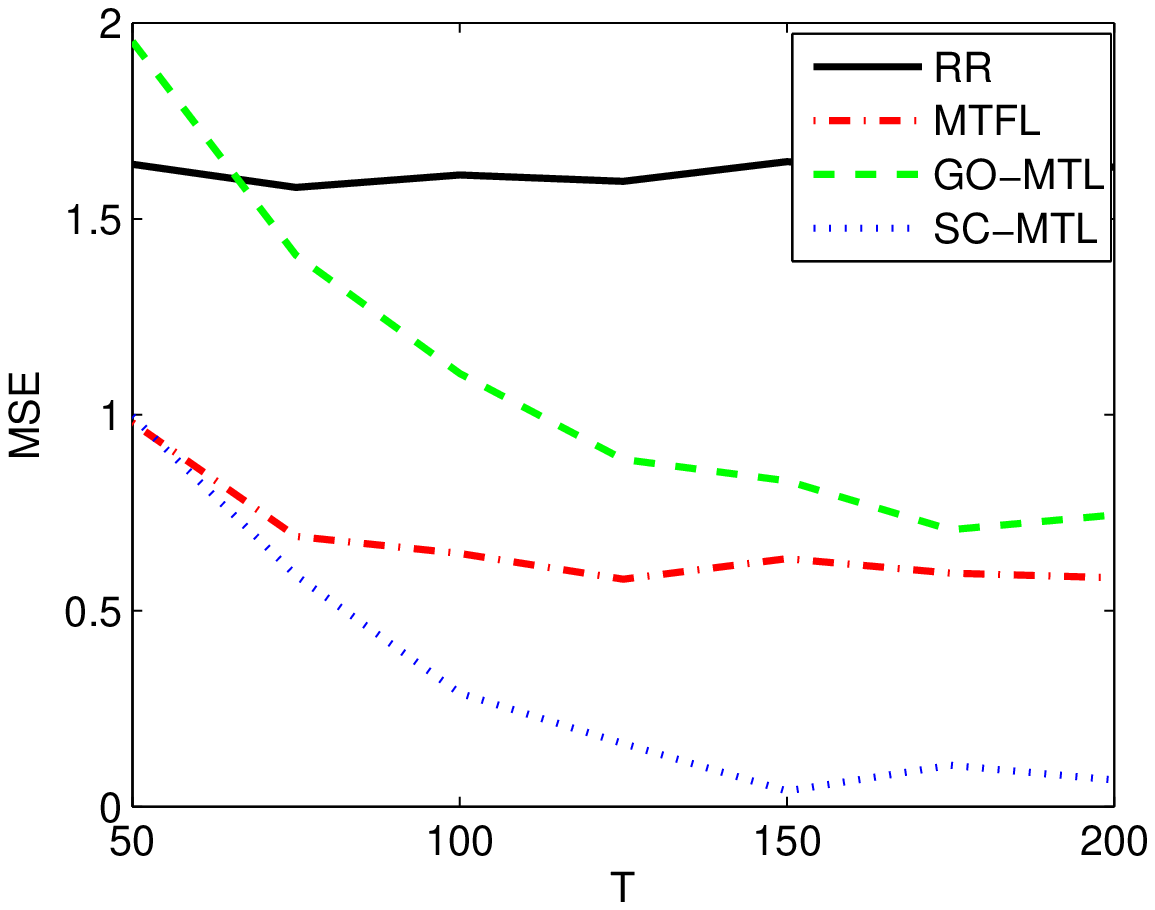}
\end{center}
\vspace{-.2truecm}
\caption{Multitask error (Top) and Transfer error (Bottom) vs. number of
training tasks $T$.}
\label{fig:1}
\end{figure}

\begin{figure}[th]
\begin{center}
\includegraphics[width=0.38\textwidth]{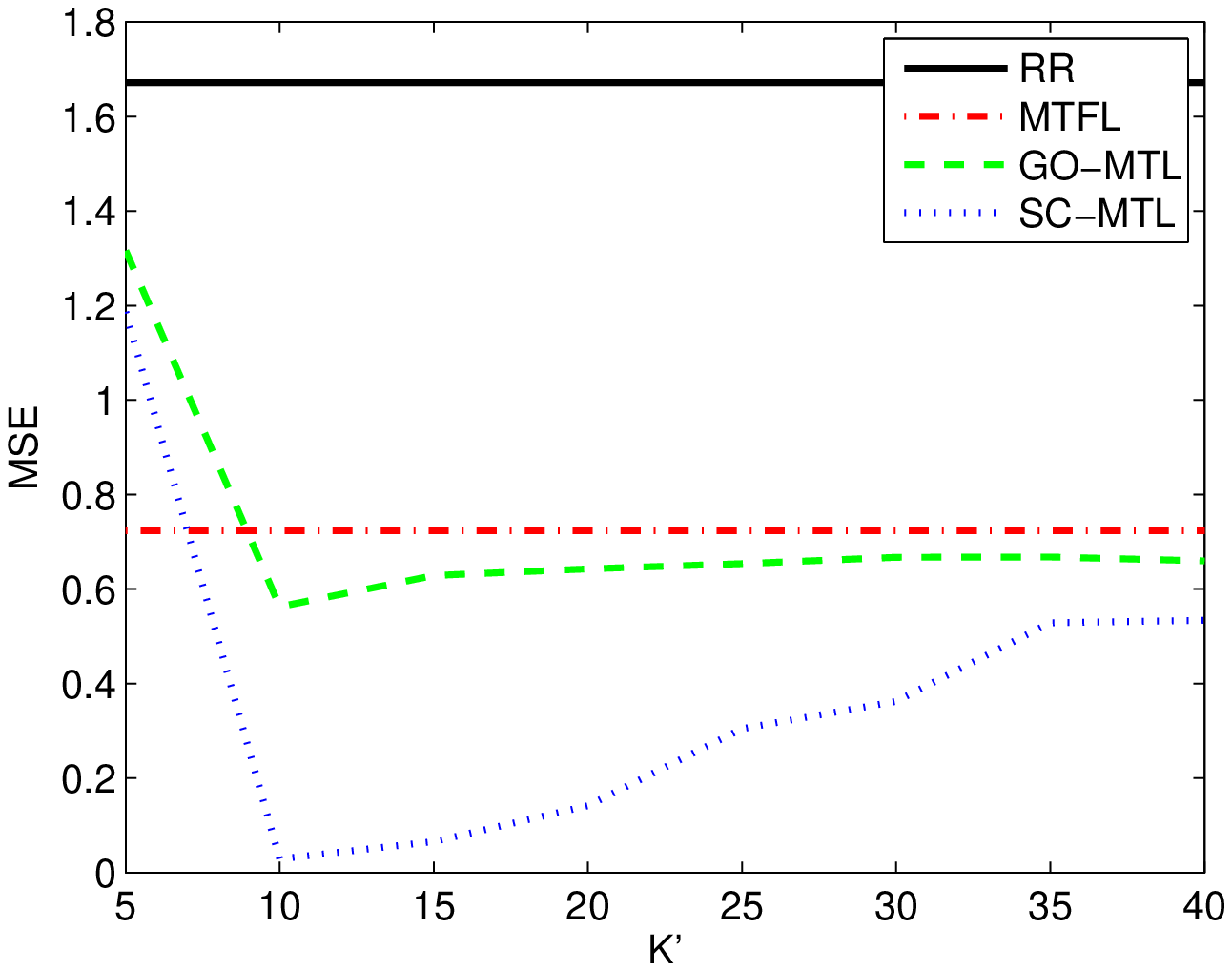}
\vspace{-.48truecm}
\includegraphics[width=0.38\textwidth]{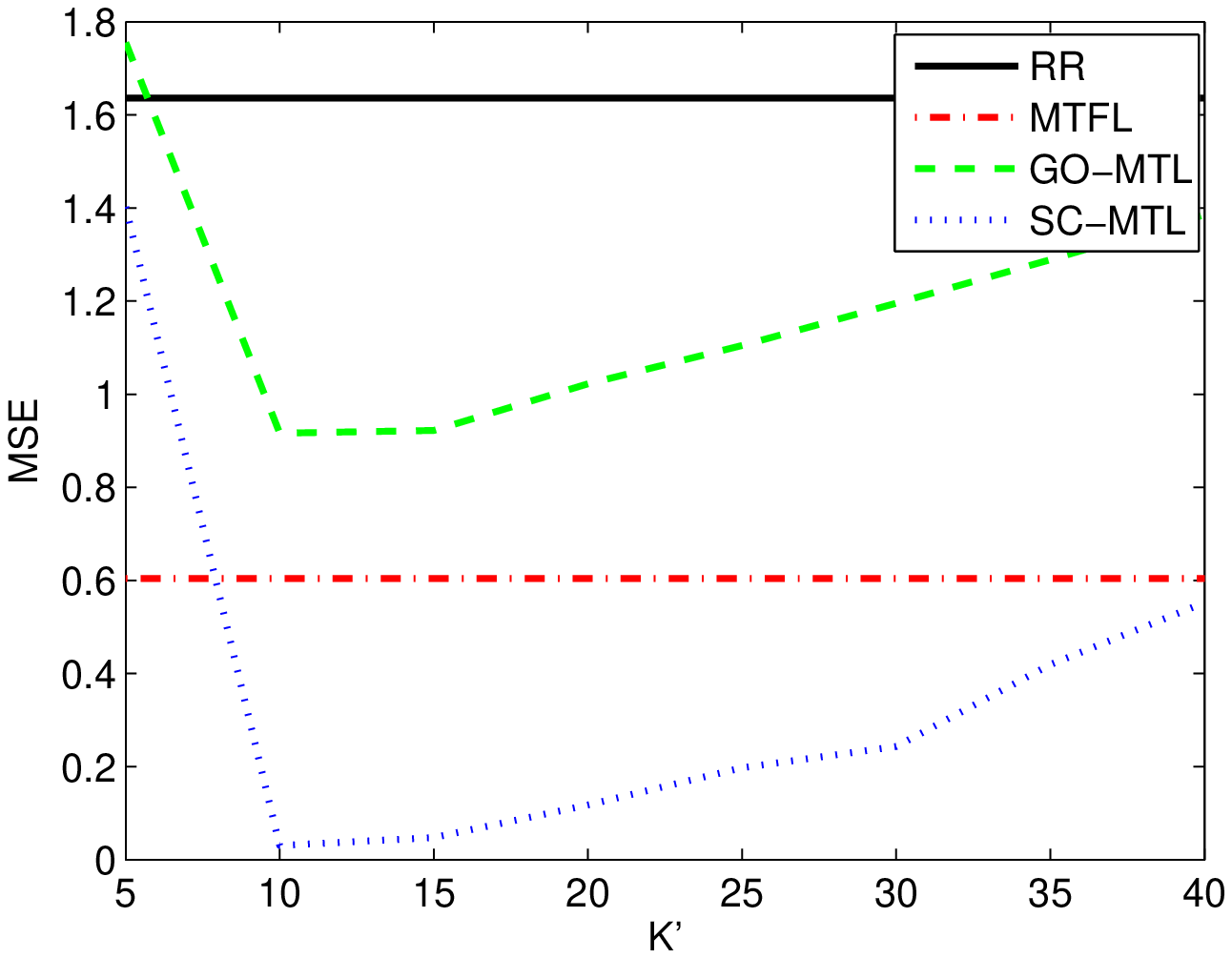}
\end{center}
\vspace{-.2truecm}
\caption{Multitask error (Top) and Transfer error (Bottom) vs. number of
atoms $K^{\prime }$ used by dictionary-based methods.}
\label{fig:2}
\end{figure}

\begin{figure}[th]
\begin{center}
\includegraphics[width=0.358\textwidth]{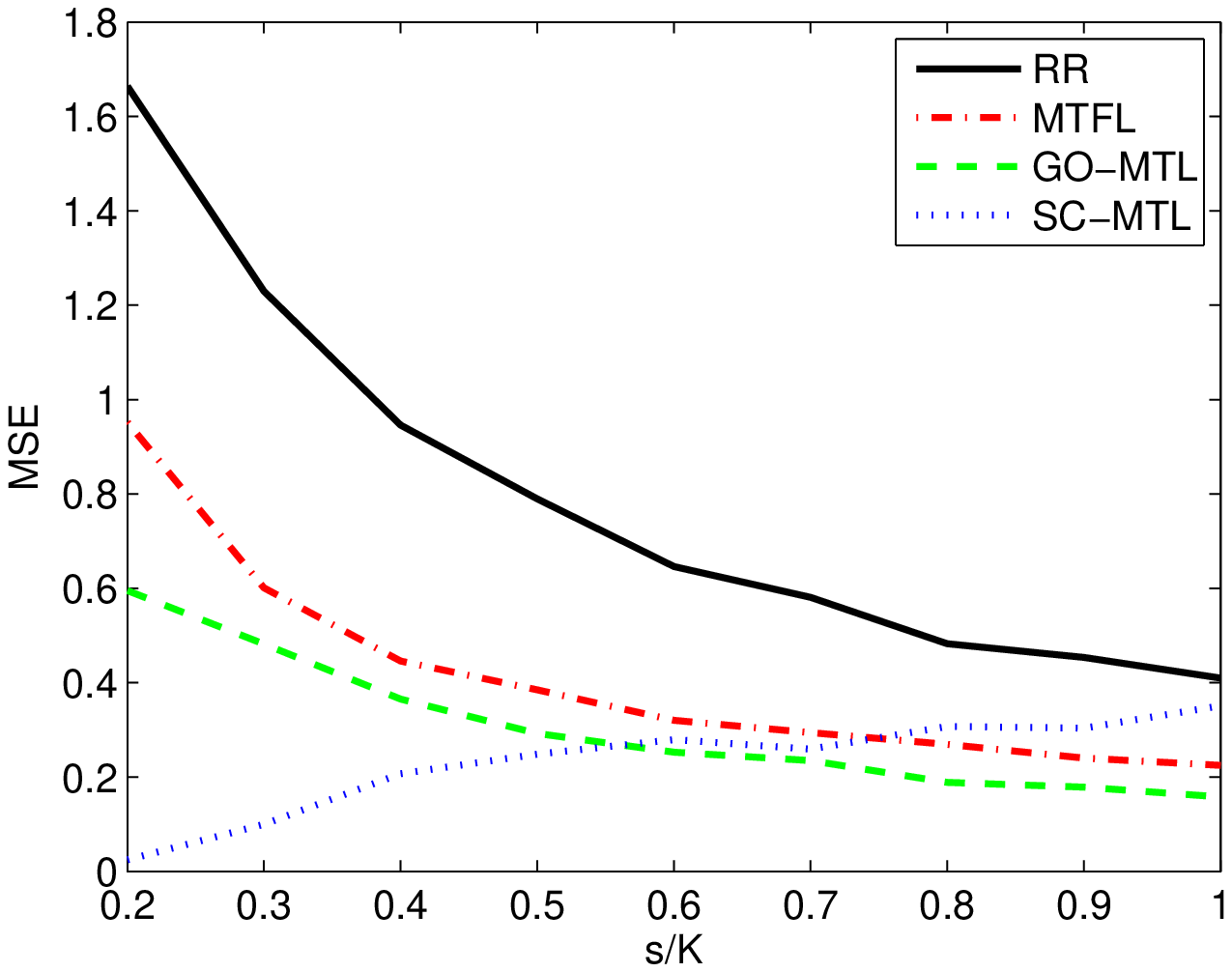}
\vspace{-.48truecm}
\includegraphics[width=0.358\textwidth]{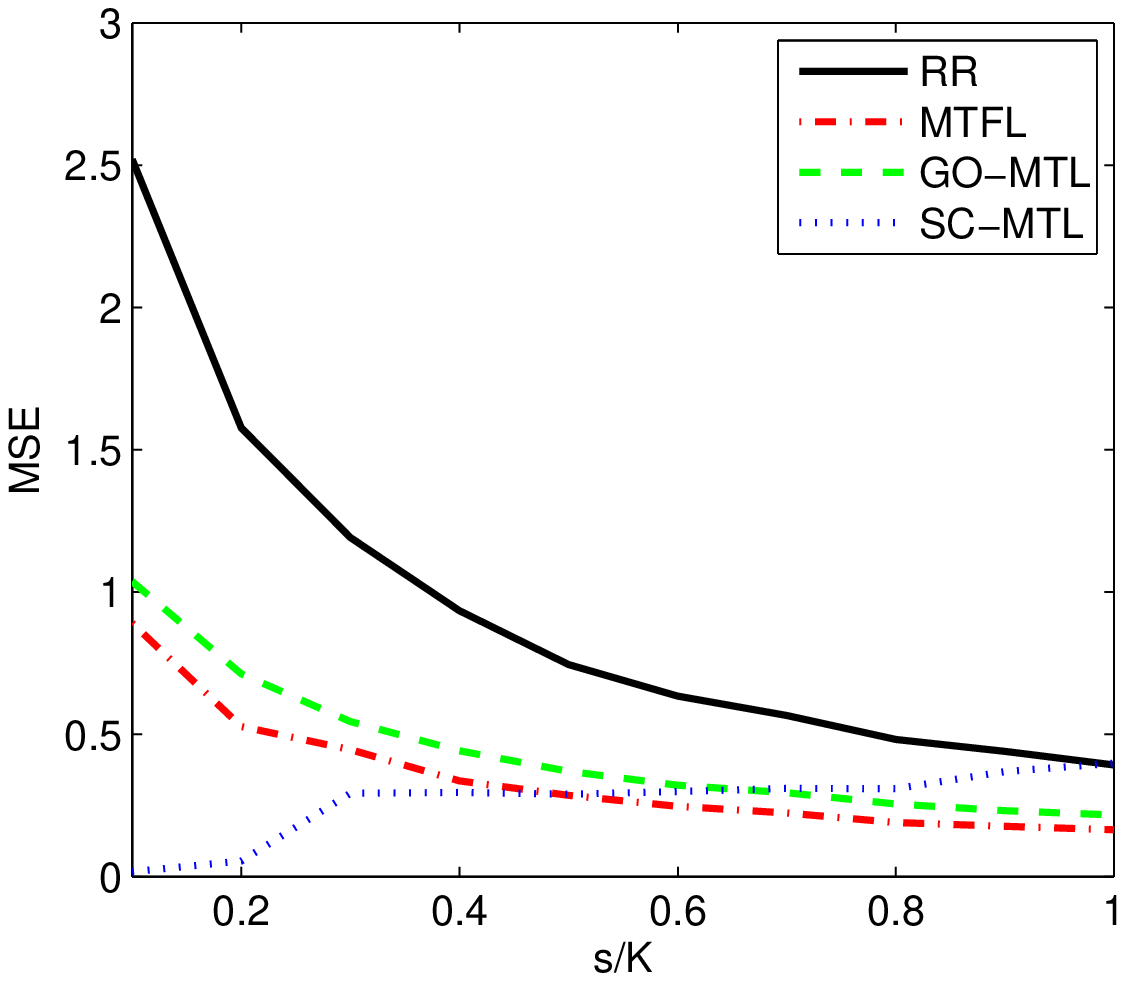}
\end{center}
\vspace{-.2truecm}
\caption{Multitask error (Top) and Transfer error (Bottom) vs. sparsity
ratio $s/K$.}
\label{fig:3}
\end{figure}
We generated a synthetic environment of tasks as follows. We choose a $%
d\times K$ matrix $D$ by sampling its columns independently from the
uniform distribution on the unit sphere in ${\mathbb{R}}^{d}$. Once $D$
is created, a generic task in the environment is given by $w=D\gamma $,
where $\gamma $ is an $s$-sparse vector obtained as follows. First, we
generate a set $J\subseteq \{1,\dots ,K\}$ of cardinality $s$, whose
elements (indices) are sampled uniformly without replacement from the set $\{%
1,\dots ,K\}$. We then set $\gamma _{j}=0$ if $j\notin J$ and otherwise
sample $\gamma _{j}\sim \mathcal{N}(0,0.1)$. Finally, we normalize $\gamma $
so that it has $\ell _{1}$-norm equal to some prescribed value $\alpha $.
Using the above procedure we generated $T$ tasks $w_{t}=D\gamma _{t}$, $%
t=1,\dots ,T$. Further, for each task $t$ we generated a training set $%
\mathbf{z}_{t}=\{(x_{ti},y_{ti})\}_{i=1}^{m}$, sampling $x_{ti}$ i.i.d. from
the uniform distribution on the unit sphere in ${\mathbb{R}}^{d}$. We
then set $y_{ti}=\langle w_{t},x_{ti}\rangle +\xi _{ti}$, with $\xi
_{ti}\sim \mathcal{N}(0,\sigma ^{2})$, where $\sigma $ is the variance of the noise.
This procedure also defines the generation of new tasks in the transfer learning experiments below.%
%We consider both low and high noise scenarios (see below). 

The above model depends on seven parameters: the number $K$ and
the dimension $d$ of the atoms, the sparsity $s$ and the $\ell _{1}$-norm $%
\alpha $ of the codes, the noise level $\sigma $, the sample size per task $m
$ and the number of training tasks $T$. In all experiments we report both
the multitask learning (MTL) and learning to learn (LTL) performance of the
methods. For MTL, we measure performance by the estimation error $%
1/T\sum_{t=1}^{T}\Vert w_{t}-{\hat{w}}_{t}\Vert ^{2}$, where ${\hat{w}}%
_{1},\dots ,{\hat{w}}_{T}$ are the estimated task vectors (in the case of SC-MTL, ${\hat{w}}_{t}={{D}({\mathbf Z})}{{\gamma}({\mathbf Z})}_{t}$ -- see the discussion in Section \ref{sec:2}. For LTL, we use the
same quantity but with a new set of tasks generated by the
environment (in the experiment below we generate $100$ new tasks). The
regularization parameter of each method is chosen by cross validation.
Finally, all experiments are repeated $50$ times, and the average
performance results are reported in the plots below.

In the first experiment, we fix $K=10,d=20,s=2,\alpha =10,m=10,\sigma =0.1$
and study the statistical performance of the methods as a function of the
number of tasks. The results, shown in Figure \ref{fig:1}, clearly indicate that the
proposed method outperforms the remaining approaches. In this experiment the number of atoms used by dictionary-based approaches, which
here we denote by $K^{\prime }$ to avoid confusion with the number of atoms $K$
of the target dictionary, was equal to $K=10$. This gives an advantage to
 both GO-MTL and SC-MTL. We therefore also studied the performance of those methods in
dependence on $K^{\prime }$. Figure \ref{fig:2}, reporting this result, is
in qualitative agreement with our theoretical analysis: the performance of
SC-MTL is not too sensitive to $K^{\prime }$ if $K^{\prime }\geq K$, and
the method still outperforms independent RR and MTFL if $K^{\prime }=4K$. On the other hand if $K^{\prime }<K$ the
performance of the method quickly degrades. In the last experiment we study
performance vs. the sparsity ratio $s/K$. Intuitively we would expect our
method to have greater advantage over MTL if $s\ll K$%
. The results, shown in Figure \ref{fig:3}, confirm this fact, also
indicating that SC-MTL is outperformed by both GO-MTL and MTFL as sparsity
becomes less pronounced ($s/K>0.6$).

\subsection{Learning to learn optical character recognition}

We have conducted experiments on real data to study the performance of
our method in a learning to learn / transfer learning setting. To this end, we employed the NIST dataset%
\footnote{The NIST dataset is available at http://www.nist.gov/srd/nistsd19.cfm%
}, which is composed of a set of $14\times14$ pixels images of handwritten
characters (digits and lower and capital case letters, for a total of 52 characters).

We considered the following experimental protocol. First, a set of $20$ characters are chosen
randomly as well as $n$ instances for each character. These are
used to learn all possibilities of $1$-vs-$1$ train tasks, which makes
$T = 190$, each of which having $m=2n$ instances.
The knowledge learned in this stage is employed to learn another set of target tasks. In our approach,
the assumption that is made is that some of the components in the dictionary learned from the training tasks, 
can also be useful for representing the target tasks. 
In order to create the target tasks, another set of $10$ characters are chosen among the remaining set of characters in the dataset, inducing a set of $45$ $1$-vs-$1$ classification tasks.
Since we are interested in the case where the training set size of the
target tasks is small, we sample only $3$ instances for each character, hence $6$ examples per task. 

In order to tune the hyperparameters of all compared approaches, we have also created another set of $45$ validation tasks by following
the process previously described, simulating the target set of tasks. Note that there is not overlapping between the digits associated to the train,
target and validation tasks.

We have run $50$ trials of the above process for different values
of $m$ and the average multiclass accuracy on the target tasks is reported in Figure 
\ref{fig:RealData_LTL}.

\iffalse
We considered the following experimental protocol. First, a set of $20$ characters are chosen
randomly as well as $n$ instances for each character. These are
used to learn all possibilities of $1$-vs-$1$ train tasks, which makes
$T=T_{\rm train} = 190$, each of which having $m=m_{\rm train}=2n$ instances.
The knowledge learned in this stage is employed to learn another set of target tasks. In our approach,
the assumption that is made is that some of the components in the dictionary learned from the training tasks, 
can also be useful for representing the target tasks. 
In order to create the target tasks, another set of $10$ characters are chosen among the remaining set of characters in the dataset, inducing a set of $T_{\rm target}=45$ $1$-vs-$1$ classification tasks.
Since we are interested in the case where the training set size of the
target tasks is small, we sample only $3$ instances for each character, hence $m_{\rm target}=6$. 

In order to tune the hyperparameters of all compared approaches, we have also created another set of validation tasks by following
the process previously described, setting $T_{\rm validation}=45$ and
$m_{\rm validation}=6$, simulating the target set of tasks. Note that
there is not overlapping between the digits associated to the train,
target and validation tasks.

We have run $50$ trials of the above process for different values
of $m_{\rm train}$ and the average multiclass accuracy on the target tasks is reported in Figure 
\ref{fig:RealData_LTL}.
\fi

\begin{figure}[t]
\begin{centering}
\includegraphics[width=0.7\columnwidth]{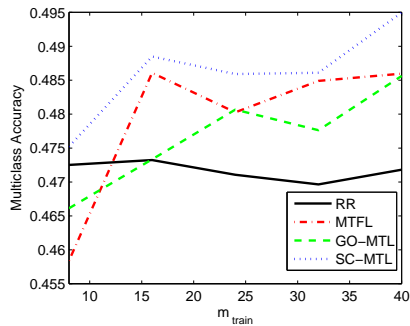}
\par\end{centering}
\vspace{-.2truecm}
\caption{\label{fig:RealData_LTL} Multiclassification accuracy of RR, MTFL
GO-MTL and SC-MTL vs. the number of training instances in the transfer
tasks, $m$. }
\end{figure}

\subsection{Sparse coding of images with missing pixels}
\label{sec:Real experiments}

In the last experiment we consider a sparse coding problem \cite{Ols2} of optical character images, with missing pixels.
We employ the Binary Alphadigits dataset\footnote{Available at {\em http://www.cs.nyu.edu/~roweis/data.html}.}, which is composed of a set of binary $20 \times16$ images of all digits and capital letters (39 images for each character). In the following experiment only the digits are used. 
%Each image is treated as a weight vector task and only a subset of $m$ pixels, chosen at random, are known for each image.
We regard each image as a task, hence the input  space is the set of $320$ possible pixels indices, while the output space is the real interval $[0,1]$, representing the gray level. We sample $T = 100,130,160,190, 220, 250$ images, equally divided among the $10$ possible digits. For each of these, a corresponding random set of $m=160$ pixel values are sampled (so the set of sample pixels varies from one image to another).

\begin{figure}[t]
\begin{center}
\includegraphics[width=0.358\textwidth]{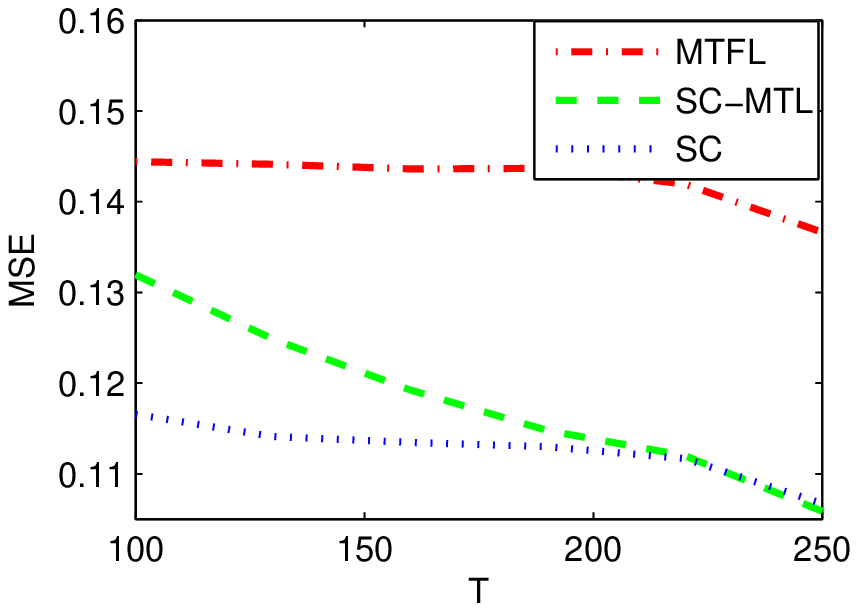}
\vspace{-.48truecm}
\includegraphics[width=0.358\textwidth]{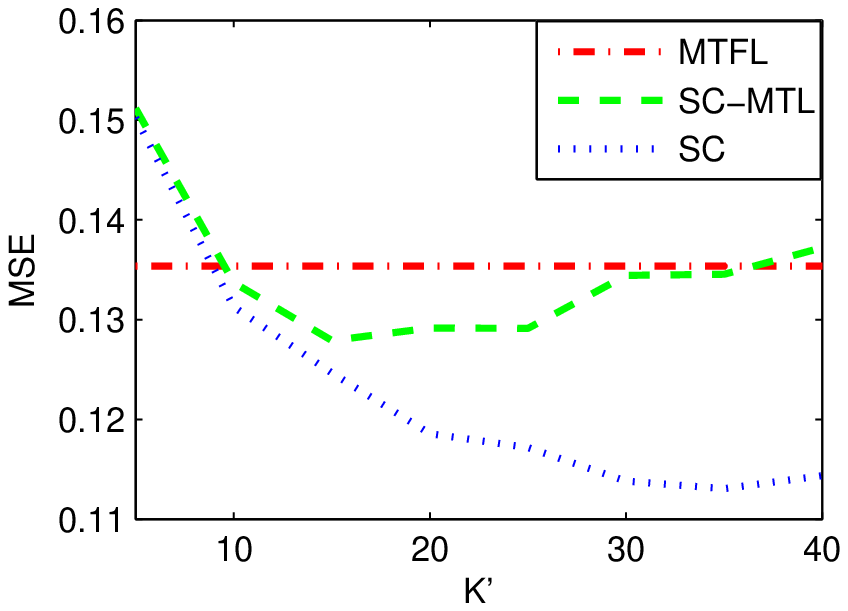}
\end{center}
\vspace{-.2truecm}
\caption{\label{fig:pixels}Transfer error vs. number of tasks T (Top) and vs. number of atoms K (Bottom) on the Binary Alphadigits dataset.}
\end{figure}

We test the performance of the dictionary learned by method \eqref{basic algorithm} in a learning to learn setting, by choosing $100$ new images. The regularization parameter for each approach is tuned using cross validation.
The results, shown in Figure \ref{fig:pixels}, indicate some advantage of the proposed method over trace norm regularization. A similar trend, not reported here due to space constraints, is obtained in the multitask setting. Ridge regression performed significantly worse and is not shown in the figure.  We also show as a reference the performance of sparse coding (SC) applied when all pixels are known.

With the aim of analyzing the atoms learned by the algorithm, we have carried out another experiment where we assume that there are $10$ underlying atoms (one for each digit). We compare the resultant dictionary to that obtained by sparse coding, where all pixels are known. The results are shown in Figure \ref{fig:dictionaries}.
%MASSI and the similarity of the dictionaries confirms the theoretical findings of Section \ref{sec:Sparse Coding}.

\begin{figure}[th]
\begin{center}
\includegraphics[width=0.458\textwidth]{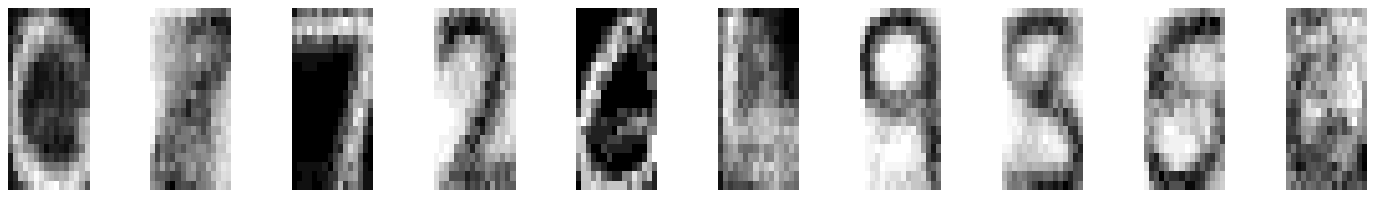}
\includegraphics[width=0.458\textwidth]{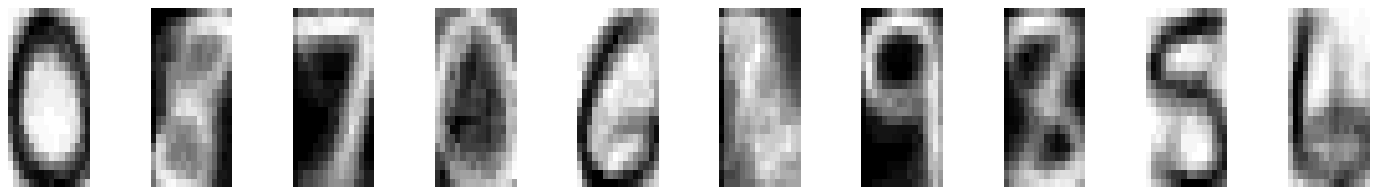}
\end{center}
\vspace{-.2truecm}
\caption{\label{fig:dictionaries} Dictionaries found by SC-MTL using $m=240$ pixels (missing $25\%$ pixels) per image (top) and by Sparse Coding employing all pixels (bottom).}
\end{figure}

\section{Summary}

\label{sec:discussion} In this paper, we have explored an
application of sparse coding, which has been widely used in
unsupervised learning and signal processing, to the domains
of multitask learning and learning to learn. Our learning bounds provide a
justification of this method and offer insights into its advantage over
independent task learning and learning dense representation of the tasks.
The bounds, which hold in a Hilbert space setting, depend on data dependent
quantities which measure the intrinsic dimensionality of the data. Numerical
simulations presented here indicate that sparse coding is a promising approach to multitask
learning and can lead to significant improvements over competing methods.

In the future, it would be valuable to study extensions of our analysis to
more general classes of code vectors. For example, we could use code sets $%
\mathcal{C}_{\alpha }$ which arise from structured sparsity norms, such as
the group Lasso, see e.g. \cite{Jenatton,Lounici}
%norm with overlapping groups 
or other families of regularizers. 
A concrete example which comes to mind is to choose $K=Qr$, $%
Q,r\in {\mathbb{N}}$ and a partition ${\mathcal{J}}=\{\{(q-1)r+1,\dots
,qr\}:q=1,\dots ,Q\}$ of the index set $\{1,\dots ,K\}$ into contiguous
index sets of size $r$. Then using a norm of the type $\Vert \gamma \Vert
=\Vert \gamma \Vert _{1}+\sum_{J\in \mathcal{J}}\Vert \gamma _{J}\Vert _{2}$
will encourage codes which are sparse and use only few of the groups in ${%
\mathcal{J}}$. Using the ball associated with this norm as our set of codes
would allow to model sets of tasks which are divided into groups.
A further natural extension of our method is nonlinear dictionary learning in which the dictionary columns correspond to functions in a reproducing kernel Hilbert space and the tasks are expressed as sparse linear combinations of such functions.
%We postpone this the study of this setting to future work.

\section*{Acknowledgments} 
This work was supported in part by EPSRC Grant EP/H027203/1 and Royal Society International Joint
Project Grant 2012/R2.

%\clearpage
%
%\newpage

%\bibliographystyle{icml2013}
%\bibliography{icml}

\clearpage

\newpage

\begin{center}
{\Large \bf Appendix}
\end{center}

In this appendix, we present the proof of Theorems \ref{Theorem Multitask} and \ref{Theorem Main}. We begin by introducing some more notation and auxiliary results. 

\appendix
\section{Notation and tools}

Issues of measurability will be ignored throughout, in particular, if $%
\tciFourier $ is a class of real valued functions on a domain $\mathcal{X}$
and $X$ a random variable with values in $\mathcal{X}$ then we will always
write ${\mathbb E}\sup_{f\in \tciFourier }f\left( X\right) $ to mean $\sup \left\{
{\mathbb E}\max_{f\in \tciFourier _{0}}f\left( X\right) :\tciFourier _{0}\subseteq
\tciFourier \text{, }\tciFourier _{0}\text{ finite}\right\} $.

In the sequel $H$ denotes a finite or infinite dimensional Hilbert space
with inner product $\left\langle \cdot ,\cdot\right\rangle $ and norm $\left\Vert
\cdot \right\Vert $. If $T$ is a bounded linear operator on $H$ its operator norm
is written $\left\Vert T\right\Vert _{\infty }=\sup \left\{ \left\Vert
Tx\right\Vert :\left\Vert x\right\Vert =1\right\}$.

Members of $H$ are denoted with lower case italics such as $x,v,w$, vectors
composed of such vectors are in bold lower case, i.e. $\mathbf{x}=\left(
x_{1},\ldots,x_{m}\right) $ or $\mathbf{v=}\left( v_{1},\ldots,v_{n}\right) $,
where $m$ or $n$ are explained in the context.

Let $B$ be the unit ball in $H$. An \textit{example} is a pair $z=\left( x,y\right) \in B\times 
%TCIMACRO{\U{211d} }%
%BeginExpansion
\mathbb{R}
%EndExpansion
=:\mathcal{Z}$, a sample is a vector of such pairs $\mathbf{z}=\left(
z_{1},\ldots,z_{m}\right) =\left( \left( x_{1},y_{1}\right) ,\ldots,\left(
x_{m},y_{m}\right) \right) $. Here we also write $\mathbf{z}=\left( \mathbf{x%
},\mathbf{y}\right) $, with $\mathbf{x}=\left( x_{1},\ldots,x_{m}\right) \in
H^{m}$ and $\mathbf{y}=\left( y_{1},\ldots,y_{m}\right) \in 
%TCIMACRO{\U{211d} }%
%BeginExpansion
\mathbb{R}
%EndExpansion
^{m}$.

A multisample is a vector $\mathbf{Z}=\left( \mathbf{z}_{1},\ldots,\mathbf{z}%
_{T}\right) $ composed of samples. We also write $\mathbf{Z}=\left( \mathbf{X%
},\mathbf{Y}\right) $ with $\mathbf{X=}\left( \mathbf{x}_{1},\ldots,\mathbf{x}%
_{T}\right) $.

For members of $%
%TCIMACRO{\U{211d} }%
%BeginExpansion
\mathbb{R}
%EndExpansion
^{K}$ we use the greek letters $\gamma $ or $\beta $. Depending on context
the inner product and euclidean norm on $%
%TCIMACRO{\U{211d} }%
%BeginExpansion
\mathbb{R}
%EndExpansion
^{K}$ will also be denoted with $\left\langle \cdot ,\cdot\right\rangle $ and $%
\left\Vert .\right\Vert $. The $\ell_1$-norm $\left\Vert \cdot \right\Vert _{1}$ on $%
%TCIMACRO{\U{211d} }%
%BeginExpansion
\mathbb{R}
%EndExpansion
^{K}$ is defined by $\left\Vert \beta \right\Vert _{1}=\sum_{k=1}^{K}\left\vert \beta_{k}\right\vert$. 

In the sequel we denote with $\mathcal{C}_{\alpha }$ the set $\left\{ \beta \in 
%TCIMACRO{\U{211d} }%
%BeginExpansion
\mathbb{R}
%EndExpansion
^{K}:\left\Vert \beta \right\Vert _{1}\leq \alpha \right\}$, abbreviate $%
\mathcal{C}$ for the $\ell _{1}$-unit ball $\mathcal{C}_{1}$. The canonical
basis of $%
%TCIMACRO{\U{211d} }%
%BeginExpansion
\mathbb{R}
%EndExpansion
^{K}$ is denoted $e_{1},\ldots,e_{K}$. Unless otherwise specified the summation
over the index $i$ will always run from $1$ to $m$, $t$ will run from $1$ to $T$, and 
$k$ will run from $1$ to $K$.

\subsection{Covariances}

For $\mathbf{x\in }H^{m}$ the empirical covariance operator $\hat{\Sigma}%
\left( \mathbf{x}\right) $ is specified by%
\begin{equation*}
\left\langle \hat{\Sigma}\left( \mathbf{x}\right) v,w\right\rangle =\frac{1}{%
m}\sum_{i}\left\langle v,x_{i}\right\rangle \left\langle
x_{i},w\right\rangle ,\text{ }v,w\in H\text{.}
\end{equation*}%
The definition implies the inequality%
\begin{equation}
\sum_{i}\left\langle v,x_{i}\right\rangle ^{2}=m\left\langle \hat{\Sigma}%
\left( \mathbf{x}\right) v,v\right\rangle \leq m\left\Vert \hat{\Sigma}%
\left( \mathbf{x}\right) \right\Vert _{\infty }\left\Vert v\right\Vert ^{2}%
\text{.}  \label{Empirical covariance useful identity}
\end{equation}%
It also follows that ${\rm tr}\left( \hat{\Sigma}\left( \mathbf{x}\right) \right)
=\left( 1/m\right) \sum_{i}\left\Vert x_{i}\right\Vert ^{2}$.

For a multisample $\mathbf{X}\in H^{mT}$ we will consider two quantities
defined in terms of the empirical covariances.%
\begin{eqnarray*}
S_{1}\left( \mathbf{X}\right)  &=& \frac{1}{T} \sum_{t}\left\Vert \hat{%
\Sigma}\left( \mathbf{x}_{t}\right) \right\Vert _{1}:= \frac{1}{T}
\sum_{t}{\rm tr}\left( \hat{\Sigma}\left( \mathbf{x}_{t}\right) \right)  \\
S_{\infty }\left( \mathbf{X}\right)  &=& \frac{1}{T} \sum_{t}\left\Vert 
\hat{\Sigma}\left( \mathbf{x}_{t}\right) \right\Vert _{\infty }:=
\frac{1}{T} \sum_{t}\lambda _{\max }\left( \hat{\Sigma}\left( \mathbf{x}%
_{t}\right) \right)
\end{eqnarray*}%
where $\lambda _{\max }$ is the largest eigenvalue. If all data points $%
x_{ti}$ lie in the unit ball of $H$ then $S_{1}\left( \mathbf{X}\right) \leq
1$. Of course $S_{1}\left( \mathbf{X}\right) $ can also be written as the
trace of the total covariance $\left( 1/T\right) \sum_{t}\hat{\Sigma}\left( 
\mathbf{x}_{t}\right) $, while $S_{\infty }\left( \mathbf{X}\right) $ will
always be at least as large as the largest eigenvalue of the total
covariance. We always have $S_{\infty }\left( \mathbf{X}\right) \leq
S_{1}\left( \mathbf{X}\right) $, with equality only if the data is
one-dimensional for all tasks. The quotient $S_{1}\left( \mathbf{X}\right)
/S_{\infty }\left( \mathbf{X}\right) $ can be regarded as a crude measure of
the effective dimensionality of the data. If the data have a high
dimensional distribution for each task then $S_{\infty }\left( \mathbf{X}\right) $
can be considerably smaller than $S_{1}\left( \mathbf{X}\right) $.

\subsection{Concentration inequalities}

Let $\mathcal{X}$ be any space. For $\mathbf{x}\in \mathcal{X}^{n}$, $1\leq
k\leq n$ and $y\in \mathcal{X}$ we use $\mathbf{x}_{k\leftarrow y}$ to
denote the object obtained from $\mathbf{x}$ by replacing the $k$-th
coordinate of $\mathbf{x}$ with $y$. That is 
\begin{equation*}
\mathbf{x}_{k\leftarrow y}=\left( x_{1},\dots ,x_{k-1},y,x_{k+1},\dots
,x_{n}\right) \text{.}
\end{equation*}%
The concentration inequality in part (i) of the following theorem, known as
the bounded difference inequality is given in \cite{McDiarmid1998}. A proof
of inequality (ii) is given in \cite{Maurer2006}.

\begin{theorem}
\label{Theorem Concentration}Let $F:\mathcal{X}^{n}\rightarrow 
%TCIMACRO{\U{211d} }%
%BeginExpansion
\mathbb{R}
%EndExpansion
$ and define $A$ and $B$ by%
\begin{eqnarray*}
A^{2} &=&\sup_{\mathbf{x}\in \mathcal{X}^{n}}\sum_{k=1}^{n}\sup_{y_{1},y_{2}%
\in \mathcal{X}}\left( F\left( \mathbf{x}_{k\leftarrow y_{1}}\right)
-F\left( \mathbf{x}_{k\leftarrow y_{2}}\right) \right) ^{2} \\
B^{2} &=&\sup_{\mathbf{x}\in \mathcal{X}^{n}}\sum_{k=1}^{n}\left( F\left( 
\mathbf{x}\right) -\inf_{y\in \mathcal{X}}F\left( \mathbf{x}_{k\leftarrow
y}\right) \right) ^{2}.
\end{eqnarray*}%
Let $\mathbf{X}=\left( X_{1},\dots ,X_{n}\right) $ be a vector of
independent random variables with values in $\mathcal{X}$, and let $\mathbf{X%
}^{\prime }$ be i.i.d. to $\mathbf{X}$. Then for any $s>0$

(i) $\Pr \left\{ F\left( \mathbf{X}\right) >{{\mathbb{E}}}F\left( \mathbf{X}%
^{\prime }\right) +s\right\} \leq e^{-2s^{2}/A^{2}};$

(ii) $\Pr \left\{ F\left( \mathbf{X}\right) >{{\mathbb{E}}}F\left( \mathbf{X}%
^{\prime }\right) +s\right\} \leq e^{-s^{2}/\left( 2B^{2}\right) }.$
\end{theorem}

\subsection{Rademacher and Gaussian averages\label{subsection rademacher and
gaussian averages}}

We will use the term \textit{Rademacher variables} for any set of
independent random variables, uniformly distributed on $\left\{ -1,1\right\} 
$, and reserve the symbol $\sigma $ for Rademacher variables. A set of
random variables is called \textit{orthogaussian} if the members are
independent ${\mathcal N}\left( 0,1\right) $-distributed (standard normal)
variables and reserve the letter $\zeta $ for standard normal variables.
Thus $\sigma _{1},\sigma _{2},\ldots,\sigma _{i},\ldots,\sigma _{11},\ldots,\sigma
_{ij}$ etc. will always be independent Rademacher variables and $\zeta
_{1},\zeta _{2},\ldots,\zeta _{i},\ldots,\zeta _{11},\ldots,\zeta _{ij}$ will always
be orthogaussian.

For $A\subseteq 
%TCIMACRO{\U{211d} }%
%BeginExpansion
\mathbb{R}
%EndExpansion
^{n}$ we define the Rademacher and Gaussian averages of $A$ \cite{Ledoux1991,Bartlett2002} as 
\begin{eqnarray*}
\mathcal{R}\left( A\right) &=&\mathbb{E}_{\sigma }\sup_{\left(
x_{1},\ldots,x_{n}\right) \in A}\frac{2}{n}\sum_{i=1}^{n}\sigma _{i}x_{i}\text{%
, } \\
\mathcal{G}\left( A\right) &=&\mathbb{E}_{\zeta }\sup_{\left(
x_{1},\ldots,x_{n}\right) \in A}\frac{2}{n}\sum_{i=1}^{n}\zeta _{i}x_{i}.
\end{eqnarray*}%
If $\mathcal{F}$ is a class of real valued functions on a space $\mathcal{X}$ and $%
\mathbf{x}=\left( x_{1},\ldots,x_{n}\right) \in \mathcal{X}^{n}$ we write 
\begin{align*}
\mathcal{F}\left( \mathbf{x}\right) =\mathcal{F}\left(
x_{1},\ldots,x_{n}\right) \\
=\left\{ \left( f\left( x_{1}\right) ,\ldots,f\left(
x_{n}\right) \right) :f\in \mathcal{F}\right\} \subseteq 
%TCIMACRO{\U{211d} }%
%BeginExpansion
\mathbb{R}
%EndExpansion
^{n}.
\end{align*}%
The empirical Rademacher and Gaussian complexities of $\mathcal{F}$ on $%
\mathbf{x}$ are respectively $\mathcal{R}\left( \mathcal{F}\left( \mathbf{x}%
\right) \right) $ and $\mathcal{G}\left( \mathcal{F}\left( \mathbf{x}\right)
\right) $.

The utility of these concepts for learning theory comes from the following
key-result (see \cite{Bartlett2002,Koltchinskii2002}), stated here in two portions for
convenience in the sequel.

\begin{theorem}
\label{Theorem Rademacher bound Expectation version}Let $\mathcal{F}$ be a
real-valued function class on a space $\mathcal{X}$ and $\mu _{1},\ldots,\mu
_{m}$ be probability measures on $\mathcal{X}$ with product measure $\mathbf{%
\mu }=\prod_{i}\mu _{i}$ on $\mathcal{X}^{m}$. For $\mathbf{x}\in \mathcal{X}%
^{m}$ define%
\begin{equation*}
\Phi \left( \mathbf{x}\right) =\sup_{f\in \mathcal{F}}\frac{1}{m}%
\sum_{i=1}^{m}\left( \mathbb{E}_{x\sim \mu _{i}}\left[ f\left( x\right) %
\right] -f\left( x_{i}\right) \right) .
\end{equation*}%
Then $\mathbb{E}_{\mathbf{x}\sim \mathbf{\mu }}\left[ \Phi \left( \mathbf{x}%
\right) \right] \leq \mathbb{E}_{\mathbf{x}\sim \mathbf{\mu }}\mathcal{R}%
\left( \mathcal{F}\left( \mathbf{x}\right) \right) $.
\end{theorem}

\begin{proof}
For any realization $\sigma =\sigma _{1},\ldots,\sigma _{m}$ of the Rademacher
variables%
\begin{eqnarray*}
\mathbb{E}_{\mathbf{x}\sim \mathbf{\mu }}\left[ \Phi \left( \mathbf{x}%
\right) \right] \\
=\mathbb{E}_{\mathbf{x}\sim \mathbf{\mu }}\sup_{f\in 
\mathcal{F}}\frac{1}{m}\mathbb{E}_{\mathbf{x}^{\prime }\sim \mathbf{\mu }%
}\sum_{i=1}^{m}\left( f\left( x_{i}^{\prime }\right) -f\left( x_{i}\right)
\right) \\
\leq \mathbb{E}_{\mathbf{x,x}^{\prime }\sim \mathbf{\mu }\times \mathbf{%
\mu }}\sup_{f\in \mathcal{F}}\frac{1}{m}\sum_{i=1}^{m}\sigma _{i}\left(
f\left( x_{i}^{\prime }\right) -f\left( x_{i}\right) \right) ,
\end{eqnarray*}%
because of the symmetry of the measure $\mathbf{\mu }\times \mathbf{\mu }%
\left( \mathbf{x},\mathbf{x}^{\prime }\right) \mathbf{=}\prod_{i}\mu
_{i}\times \prod_{i}\mu _{i}\left( \mathbf{x},\mathbf{x}^{\prime }\right) $%
 under the interchange $x_{i}\leftrightarrow x_{i}^{\prime }$. Taking the
expectation \ in $\sigma $ and applying the triangle inequality gives the
result.
\end{proof}

\begin{theorem}
\label{Rademacher Bound}Let $\mathcal{F}$ be a $\left[ 0,1\right] $-valued
function class on a space $\mathcal{X}$, and $\mathbf{\mu }$ as above. For $%
\delta >0$ we have with probability greater than $1-\delta $ in the sample $%
\mathbf{x}\sim \mathbf{\mu }$ that for all $f\in \mathcal{F}$%
\begin{equation*}
\mathbb{E}_{x\sim \mathbf{\mu }}\left[ f\left( x\right) \right] \leq \frac{1%
}{m}\sum_{i=1}^{m}f\left( x_{i}\right) +\mathbb{E}_{\mathbf{x}\sim \mathbf{%
\mu }}\mathcal{R}\left( \mathcal{F}\left( \mathbf{x}\right) \right) +\sqrt{%
\frac{\ln \left( 1/\delta \right) }{2m}}.
\end{equation*}
\end{theorem}

To prove this we apply the bounded-difference inequality ( part (i) of Theorem %
\ref{Theorem Concentration}) to the function $\Phi $ of the previous theorem
(see e.g. \cite{Bartlett2002}). Under the conditions of this result,
changing one of the $x_{i}$ will not change $\mathcal{R}\left( \mathcal{F}%
\left( \mathbf{x}\right) \right) $ by more than $2$, so again by the bounded
difference inequality applied to $\mathcal{R}\left( \mathcal{F}\left( 
\mathbf{x}\right) \right) $ and a union bound we obtain the data dependent
version

\begin{corollary}
\label{Corollary empirical Rademacher bound}Let $\mathcal{F}$ and $\mathbf{%
\mu }$ be as above. For $\delta >0$ we have with probability greater than $%
1-\delta $ in the sample $\mathbf{x}\sim \mathbf{\mu }$ that for all $f\in 
\mathcal{F}$%
\begin{equation*}
\mathbb{E}_{x\sim \mathbf{\mu }}\left[ f\left( x\right) \right] \leq \frac{1%
}{m}\sum_{i=1}^{m}f\left( x_{i}\right) +\mathcal{R}\left( \mathcal{F}\left( 
\mathbf{x}\right) \right) +\sqrt{\frac{9\ln \left( 2/\delta \right) }{2m}}.
\end{equation*}
\end{corollary}

To bound Rademacher averages the following result is very useful \cite{Bartlett2002,Zhang,Ledoux1991}

\begin{lemma}
\label{Lemma Rademacher Lipschitz}Let $A\subseteq 
%TCIMACRO{\U{211d} }%
%BeginExpansion
\mathbb{R}
%EndExpansion
^{n}$, and let $\psi _{1},\ldots,\psi _{n}$ be real functions such that $\psi
_{i}\left( s\right) -\psi _{i}\left( t\right) \leq L\left\vert
s-t\right\vert $,$\forall i$, and $s,t\in 
%TCIMACRO{\U{211d} }%
%BeginExpansion
\mathbb{R}
%EndExpansion
$. Define $\mathbf{\psi }\left( A\right) =\left\{ \psi _{1}\left(
x_{1}\right) ,\ldots,\psi _{n}\left( x_{n}\right) :\left(
x_{1},\ldots,x_{n}\right) \in A\right\} $. Then%
\begin{equation*}
\mathcal{R}\left( \mathbf{\psi }\left( A\right) \right) \leq L\mathcal{R}%
\left( A\right) \text{.}
\end{equation*}
\end{lemma}

Sometimes it is more convenient to work with gaussian averages which can be
used instead, by virtue of the next lemma. For a proof see e.g. \cite{Ledoux1991}

\begin{lemma}
\label{Lemma Gauss dominates Rademacher}For $A\subseteq 
%TCIMACRO{\U{211d} }%
%BeginExpansion
\mathbb{R}
%EndExpansion
^{k}$ we have $\mathcal{R}\left( A\right) \leq \sqrt{\pi /2}~\mathcal{G}%
\left( A\right) $.
\end{lemma}

The next result is known as Slepian's lemma (\cite{Slepian}, \cite{Ledoux1991}).

\begin{theorem}
\label{Slepian Lemma}Let $\Omega $ and $\Xi $ be mean zero, separable
Gaussian processes indexed by a common set $\mathcal{S}$, such that%
\begin{equation*}
\mathbb{E}\left( \Omega _{s_{1}}-\Omega _{s_{2}}\right) ^{2}\leq \mathbb{E}%
\left( \Xi _{s_{1}}-\Xi _{s_{2}}\right) ^{2}\text{ for all }s_{1},s_{2}\in 
\mathcal{S}\text{.}
\end{equation*}%
Then%
\begin{equation*}
\mathbb{E}\sup_{s\in \mathcal{S}}\Omega _{s}\leq \mathbb{E}\sup_{s\in 
\mathcal{S}}\Xi _{s}.
\end{equation*}
\end{theorem}

\section{Proofs}

\subsection{Multitask learning}
\label{sec:AMTL}
In this section we prove Theorem \ref{Theorem Multitask}. It is an immediate
consequence of Hoeffding's inequality and the following uniform bound on the
estimation error.

\begin{theorem}
\label{Theorem Multitask uniform bound}Let $\delta >0$, fix $K$ and let $\mu
_{1},\ldots,\mu _{T}$ be probability measures on $H\times 
%TCIMACRO{\U{211d} }%
%BeginExpansion
\mathbb{R}
%EndExpansion
$. With probability at least $1-\delta $ in the draw of $\mathbf{Z}\sim
\prod_{t=1}^{T} \mu _{t}$ we have for all $D\in \mathcal{D}_{K}
$ and all $\mathbf{\gamma }\in \mathcal{C}_\alpha^{T}\mathcal{\ }$that%
\begin{multline*}
\frac{1}{T}\sum_{t=1}^{T}{\mathbb E}_{\left( x,y\right) \sim \mu _{t}}\left[ \ell
\left( \left\langle D\gamma _{t},x\right\rangle ,y\right) \right] \\
-\frac{1}{%
mT}\sum_{t=1}^{T}\sum_{i=1}^{m}\ell \left( \left\langle D\gamma
_{t},x_{ti}\right\rangle ,y_{ti}\right)  \\
\shoveleft{\leq L\alpha \sqrt{\frac{2S_{1}\left( \mathbf{X}\right) \left( K+12\right) }{%
mT}} } \\
\shoveleft{+L\alpha \sqrt{\frac{8S_{\infty }\left( \mathbf{X}\right) \ln \left(
2K\right) }{m}}+\sqrt{\frac{9\ln 2/\delta }{2mT}} }.
\end{multline*}%
\end{theorem}

The proof of this theorem requires auxiliary results. Fix $\mathbf{X}\in
H^{mT}$ and for $\mathbf{\gamma }=\left( \gamma _{1},\ldots,\gamma _{T}\right)
\in \left( 
%TCIMACRO{\U{211d} }%
%BeginExpansion
\mathbb{R}
%EndExpansion
^{K}\right) ^{T}$ define the random variable 
\begin{equation}
F_{\mathbf{\gamma }}=F_{\mathbf{\gamma }}\left( \mathbf{\sigma }\right)
=\sup_{D\in \mathcal{D}_{K}}\sum_{t,i}\sigma _{ti}\left\langle D\gamma
_{t},x_{ti}\right\rangle .  \label{Definition of F_gamma}
\end{equation}

\begin{lemma}
\label{Lemma Multitask Rad estimate 1}
(i) If $\mathbf{\gamma }=(\gamma_1,\dots,\gamma_T)$ satisfies $\left\Vert \gamma _{t}\right\Vert \leq 1
$ for all $t$, then%
\begin{equation*}
{\mathbb E}F_{\mathbf{\gamma }}\leq \sqrt{mTK~S_{1}\left( \mathbf{X}\right) }.
\end{equation*}
(ii) If $\mathbf{\gamma }$ satisfies $\left\Vert \gamma _{t}\right\Vert
_{1}\leq 1$ for all $t$, then for any $s\geq 0$%
\begin{equation*}
\Pr \left\{ F_{\mathbf{\gamma }}\geq {\mathbb E}\left[ F_{\mathbf{\gamma }}\right]
+s\right\} \leq \exp \left( \frac{-s^{2}}{8mT~S_{\infty }\left( \mathbf{X}%
\right) }\right) .
\end{equation*}
\end{lemma}

\begin{proof}
(i) We observe that \\
\par ${\mathbb E}F_{\mathbf{\gamma }}=\mathbb{E}\, \underset{D}{{\rm sup}}\underset{k}{\sum}\left\langle De_{k},\underset{t,i}{\sum}\sigma_{ti}\gamma_{tk}x_{ti}\right\rangle \\$
\begin{eqnarray*}
&\leq &\sup_{D}\left( \sum_{k}\left\Vert De_{k}\right\Vert ^{2}\right)
^{1/2}{\mathbb E}\left( \sum_{k}\left\Vert \sum_{t,i}\sigma _{ti}\gamma
_{tk}x_{ti}\right\Vert ^{2}\right) ^{1/2} \\
&\leq & \sqrt{K}\left(\underset{k}{\sum}\mathbb{E}\left\Vert \underset{t,i}{\sum}\sigma_{ti}\gamma_{tk}x_{ti}\right\Vert ^{2}\right)^{1/2} \\
&= & \sqrt{K}\left(\underset{k,t,i}{\sum}\left|\gamma_{tk}\right|^{2}\left\Vert x_{ti}\right\Vert ^{2}\right)^{1/2}\\
&= &  \sqrt{K}\left(\underset{t}{\sum}\left(\underset{k}{\sum}\left|\gamma_{tk}\right|^{2}\right)\underset{i}{\sum}\left\Vert x_{ti}\right\Vert ^{2}\right)^{1/2}\\
&\leq &\sqrt{K\sum_{t,i}\left\Vert x_{ti}\right\Vert ^{2}} = \sqrt{mTK~S_{1}\left( \mathbf{X}\right) }.
\end{eqnarray*}

(ii) For any configuration $\mathbf{\sigma }$ of the Rademacher variables
let $D\left( \mathbf{\sigma }\right) $ be the maximizer in the definition of 
$F_{\mathbf{\gamma }}\left( \mathbf{\sigma }\right) $. Then for any $s\in
\left\{ 1,\ldots,T\right\} $, $j\in \left\{ 1,\ldots,m\right\} $ and any $%
\sigma ^{\prime }\in \left\{ -1,1\right\} $ to replace $\sigma _{sj}$ we
have 
\begin{equation*}
F_{\mathbf{\gamma }}\left( \mathbf{\sigma }\right) -F_{\mathbf{\gamma }%
}\left( \mathbf{\sigma }_{\left( sj\right) \leftarrow \sigma ^{\prime
}}\right) \leq 2\left\vert \left\langle D\left( \mathbf{\sigma }\right)
\gamma _{s},x_{sj}\right\rangle \right\vert .
\end{equation*}%
Using the inequality (\ref{Empirical covariance useful identity}) we then
obtain 
\par $\sum_{sj}\left( F_{\gamma }\left( \mathbf{\sigma }\right) -\inf_{\sigma
^{\prime }\in \left\{ -1,1\right\} }F_{\mathbf{\gamma }}\left( \mathbf{%
\sigma }_{\left( sj\right) \leftarrow \sigma ^{\prime }}\right) \right) ^{2}$
\begin{eqnarray*}
&\leq &4\sum_{t,i}\left\langle D\left( \mathbf{\sigma }\right) \gamma
_{t},x_{ti}\right\rangle ^{2}\\&\leq &4m\sum_{t}\left\Vert \hat{\Sigma}\left( 
\mathbf{x}_{t}\right) \right\Vert _{\infty }\left\Vert D\left( \mathbf{%
\sigma }\right) \gamma _{t}\right\Vert ^{2} \\
&\leq &4m\sum_{t}\left\Vert \hat{\Sigma}\left( \mathbf{x}_{t}\right)
\right\Vert _{\infty }.
\end{eqnarray*}%
In the last inequality we used the fact that for any $D\in \mathcal{D}_{K}$
we have $\left\Vert D\gamma _{t}\right\Vert \leq \sum_{k}\left\vert \gamma
_{tk}\right\vert \left\Vert De_{k}\right\Vert \leq \left\Vert \gamma
_{t}\right\Vert _{1}\leq 1$. The conclusion now follows from part (ii) of
Theorem \ref{Theorem Concentration}.
\end{proof}

\begin{proposition}
\label{Proposition Multi-task Rademacher Bound}For every fixed $%
\mathbf{Z}=\left( \mathbf{X,Y}\right) \in \left( H\times 
%TCIMACRO{\U{211d} }%
%BeginExpansion
\mathbb{R}
%EndExpansion
\right) ^{mT}$ we have%
\par ${\mathbb E}_{\sigma }\sup_{D\in D,\mathbf{\gamma }\in \left( \mathcal{C}_{\alpha
}\right) ^{T}}\sum_{t,i}\sigma _{it}\ell \left( \left\langle D\gamma
_{t},x_{ti}\right\rangle ,y_{ti}\right)$
\begin{equation*}
\leq L\alpha \sqrt{2mTS_{1}\left( 
\mathbf{X}\right) \left( K+12\right) }+L\alpha T\sqrt{8mS_{\infty }\left( 
\mathbf{X}\right) \ln \left( 2K\right) }.
\end{equation*}
\end{proposition}

\begin{proof}
It suffices to prove the result for $\alpha =1$, the general result being a
consequence of rescaling. By Lemma \ref{Lemma Rademacher Lipschitz} and the
Lipschitz properties of the loss function $\ell $ we have%
\par ${\mathbb E}_{\sigma }\sup_{D\in {\mathcal D}_K,\mathbf{\gamma }\in \left( \mathcal{C}
\right) ^{T},}\sum_{t,i}\sigma _{it}\ell \left( \left\langle D\gamma
_{t},x_{ti}\right\rangle ,y_{ti}\right)$
\begin{equation}
\leq L{\mathbb E}_{\sigma }\sup_{D\in {\mathcal D}_K,%
\mathbf{\gamma }\in \left( \mathcal{C}\right)
^{T},}\sum_{t,i}\sigma _{it}\left\langle D\gamma _{t},x_{ti}\right\rangle .
\label{Multitask Rademacher bound inequality 0}
\end{equation}%
Since linear functions on a compact convex set attain their maxima at the
extreme points, we have 
\begin{equation}
{\mathbb E}\sup_{D\in {\mathcal D}_K,\mathbf{\gamma }\in (\mathcal{C})^{T},}\sum_{t=1}^{T}%
\sum_{i=1}^{m}\sigma _{it}\left\langle D\gamma _{t},x_{ti}\right\rangle
={\mathbb E}\max_{\mathbf{\gamma }\in \text{ext}\left( \mathcal{C}\right) ^{T}}F_{%
\mathbf{\gamma }},  \label{Multitask Rademacher Bound inequality 1}
\end{equation}%
where $F_{\mathbf{\gamma }}$ is defined as in (\ref{Definition of F_gamma}).
Let $c=\sqrt{mKTS_{1}\left( \mathbf{X}\right) }$. %
Now for any $\delta \geq 0$ we have, since $F_{\mathbf{\gamma }}\geq 0$,%
\par ${\mathbb E}\max_{\mathbf{\gamma }\in \text{ext}\left( \mathcal{C}\right) ^{T}}F_{%
\mathbf{\gamma }} = \int_{0}^{\infty }\Pr \left\{ \max_{\mathbf{\gamma }\in 
\text{ext}\left( \mathcal{C}\right) ^{T}}F_{\mathbf{\gamma }}>s\right\} ds \\$
\begin{eqnarray*}
&\leq & c+\delta +\sum_{\mathbf{\gamma 
}\in \left( \text{ext}\left( \mathcal{C}\right) \right) ^{T}}\int_{\sqrt{%
mKTS_{1}\left( \mathbf{X}\right) }+\delta }^{\infty }\Pr \left\{ F_{\mathbf{%
\gamma }}>s\right\} ds \\
&\leq & c+\delta +\sum_{\mathbf{\gamma 
}\in \left( \text{ext}\left( \mathcal{C}\right) \right) ^{T}}\int_{\delta
}^{\infty }\Pr \left\{ F_{\mathbf{\gamma }}>{\mathbb E}F_{\mathbf{\gamma }}+s\right\}
ds \\
&\leq & c+\delta +\left( 2K\right)
^{T}\int_{\delta }^{\infty }\exp \left( \frac{-s^{2}}{8mTS_{\infty }\left( 
\mathbf{X}\right) }\right) ds \\
&\leq & c+\delta +\frac{4mTS_{\infty
}\left( \mathbf{X}\right) \left( 2K\right) ^{T}}{\delta }\exp \left( \frac{%
-\delta ^{2}}{8mTS_{\infty }\left( \mathbf{X}\right) }\right) .
\end{eqnarray*}%
Here the first inequality follows from the fact that probabilities never
exceed 1 and a union bound. The second inequality follows from Lemma \ref%
{Lemma Multitask Rad estimate 1}, part (i), since ${\mathbb E}F_{\mathbf{k}}\leq \sqrt{%
mKTS_{1}\left( \mathbf{X}\right) }$. The third inequality follows from Lemma %
\ref{Lemma Multitask Rad estimate 1}, part (ii), and the fact that the
cardinality of ext$\left( \mathcal{C}\right) $ is $2K$, and the last
inequality follows from a well known estimate on Gaussian random variables.
Setting $\delta =\sqrt{8mTS_{\infty }\left( \mathbf{X}\right) \ln \left(
e\left( 2K\right) ^{T}\right) }$ we obtain with some easy simplifying
estimates%
\par ${\mathbb E}\max_{\mathbf{\gamma }\in \text{ext}\left( \mathcal{C}\right) ^{T}}F_{%
\mathbf{\gamma }}\leq \sqrt{2mT\left( K+12\right) S_{1}\left( \mathbf{X}%
\right) }$
\begin{equation*}
+T\sqrt{8mS_{\infty }\left( \mathbf{X}\right) \ln \left( 2K\right) }%
,
\end{equation*}%
which together with (\ref{Multitask Rademacher bound inequality 0}) and (\ref%
{Multitask Rademacher Bound inequality 1}) gives the result.
\end{proof}

Theorem \ref{Theorem Multitask uniform bound} now follows from Corollary \ref%
{Corollary empirical Rademacher bound}.

If the set $\mathcal{C}_{\alpha }$ is replaced by any other subset $%
\mathcal{C}^{\prime }$ of the $\ell _{2}$-ball of radius $\alpha $, a
similar proof strategy can be employed. The denominator in the exponent of
Lemma \ref{Lemma Multitask Rad estimate 1}-(ii) then obtains another factor of $\sqrt{K}$. The union bound
over the extreme points in ext$\left( \mathcal{C}\right) $ in the previous
proposition can be replaced by a union bound over a cover $\mathcal{C}%
^{\prime }$. This leads to the alternative result mentioned in Remark 5
following the statement of Theorem \ref{Theorem Multitask}.

Another modification leads to a bound for the method presented in \cite{Daume}, 
where the constraint $\left\Vert De_{k}\right\Vert \leq 1$ is
replaced by $\left\Vert D\right\Vert _{2}\leq \sqrt{K}$ (here $\left\Vert
\cdot\right\Vert _{2}$ is the Frobenius or Hilbert Schmidt norm) and the
constraint $\left\Vert \gamma _{t}\right\Vert _{1}\leq \alpha ,\forall t$ is
replaced by $\sum \left\Vert \gamma _{t}\right\Vert _{1}\leq \alpha T$. To
explain the modification we set $\alpha =1$. Part (i) of Lemma \ref{Lemma Multitask Rad estimate 1} 
is easily verified. The union bound over $\left( \text{%
ext}\left( \mathcal{C}\right) \right) ^{T}$ in the previous proposition is
replaced by a union bound over the $2TK$ extreme points of the $\ell _{1}$%
-Ball of radius $T$ in $%
%TCIMACRO{\U{211d} }%
%BeginExpansion
\mathbb{R}
%EndExpansion
^{TK}$. For part (ii)\ we use the fact that the concentration result is only
needed for $\mathbf{\gamma }$ being an extreme point (so that it involves
only a single task) and obtain the bound $\sum_{t}\left\Vert \hat{\Sigma}%
\left( \mathbf{x}_{t}\right) \right\Vert _{\infty }\left\Vert D\gamma
_{t}\right\Vert ^{2}\leq TKS_{\infty }^{\prime }\left( \mathbf{X}\right) $,
leading to 
\[
\Pr \left\{ F_{\mathbf{\gamma }}\geq E\left[ F_{\mathbf{\gamma }}\right]
+s\right\} \leq \exp \left( \frac{-s^{2}}{8mTK~S_{\infty }^{\prime }\left( 
\mathbf{X}\right) }\right) .
\]%
Proceeding as above we obtain the excess risk bound
\par $L\alpha \sqrt{\frac{2S_{1}\left( \mathbf{X}\right) \left( K+12\right) }{mT}}%
+L\alpha \sqrt{\frac{8KS_{\infty }^{\prime }\left( \mathbf{X}\right) \ln
\left( 2KT\right) }{m}}$
\[
+\sqrt{\frac{8\ln 4/\delta }{mT}},
\]%
to replace the bound in Theorem \ref{Theorem Multitask}. The factor $\sqrt{K}
$ in the second term seems quite weak, but it must be borne in mind that the
constraint $\left\Vert D\right\Vert _{2}\leq \sqrt{K}$ is much weaker than $%
\left\Vert De_{k}\right\Vert \leq 1$, and allows for a smaller approximation
error. If we retain $\left\Vert De_{k}\right\Vert \leq 1$ and only modify
the $\gamma $-constraint to  $\sum \left\Vert \gamma _{t}\right\Vert
_{1}\leq \alpha T$ the $\sqrt{K}
$ in the second term disappears and by comparison to Theorem \ref{Theorem Multitask} 
there is only and additional $\ln T$ and the switch from $%
S_{\infty }\left( \mathbf{X}\right) $ to $S_{\infty }^{\prime }\left( 
\mathbf{X}\right) $, reflecting the fact that $\sum \left\Vert \gamma
_{t}\right\Vert _{1}\leq \alpha T$ is a much weaker constraint than $%
\left\Vert \gamma _{t}\right\Vert _{1}\leq \alpha ,\forall t$, so that,
again, a smaller minimum in (\ref{basic algorithm}) is possible for the modified method.

\subsection{Learning to learn}
\label{sec:ALTL}

In this section we prove Theorem \ref{Theorem Main}. The basic strategy is
as follows. Recall the definition (\ref{Definition RHO_EPS}) of the measure $%
\rho _{\mathcal{E}}$, which governs the generation of a training sample in
the environment $\mathcal{E}$. On a given training sample $\mathbf{z\sim }%
\rho _{\mathcal{E}}$ the algorithm $A_{D}$ as defined in (\ref{Algorithm AD}%
) incurs the empirical risk%
\begin{equation*}
\hat{R}_{D}\left( \mathbf{z}\right) =\min_{\gamma \in \mathcal{C}_\alpha}\frac{1}{m}%
\sum_{i=1}^{m}\ell \left( \left\langle D\gamma ,x_{i}\right\rangle
,y_{i}\right) .
\end{equation*}%
The algorithm $A_{D}$, essentially being the Lasso, has very good estimation
properties, so $\hat{R}_{D}\left( \mathbf{z}\right) $ will be close to the
true risk of $A_{D}$ in the corresponding task. This means that we only
really need to estimate the expected empirical risk ${\mathbb E}_{\mathbf{z\sim }\rho
_{\mathcal{E}}}\hat{R}_{D}\left( \mathbf{z}\right) $ of $A_{D}$ on future
tasks. On the other hand the minimization problem (\ref{basic algorithm})
can be written as 
\begin{equation*}
\min_{D\in \mathcal{D}_K}\frac{1}{T}\sum_{t=1}^{T}\hat{R}_{D}\left( \mathbf{z}%
_{t}\right) \text{ with }\mathbf{Z=}\left( z_{1},\ldots,z_{T}\right) \sim
\left( \rho _{\mathcal{E}}\right) ^{T},
\end{equation*}%
with dictionary $D\left( \mathbf{Z}\right) $ being the minimizer. If $%
\mathcal{D}_K$ is not too large this should be similar to ${\mathbb E}_{\mathbf{z\sim }%
\rho _{\mathcal{E}}}\hat{R}_{D\left( \mathbf{Z}\right) }\left( \mathbf{z}%
\right) $. In the sequel we make this precise.

\begin{lemma}
\label{Lemma LTL auc 1}For $v\in H$ with $\left\Vert v\right\Vert \leq 1$
and $\mathbf{x}\in H^{m}$ let $F$ be the random variable%
\begin{equation*}
F=\left\vert \left\langle v,\sum_{i}\sigma _{i}x_{i}\right\rangle
\right\vert \text{.}
\end{equation*}%
Then (i) ${\mathbb E}F\leq \sqrt{m}\left\Vert \hat{\Sigma}\left( \mathbf{x}\right)
\right\Vert _{\infty }^{1/2}$ and (ii) for $t\geq 0$%
\begin{equation*}
\Pr \left\{ F>{\mathbb E}F+s\right\} \leq \exp \left( \frac{-s^{2}}{2m\left\Vert \hat{%
\Sigma}\left( \mathbf{x}\right) \right\Vert _{\infty }}\right) .
\end{equation*}
\end{lemma}

\begin{proof}
(i). Using Jensen's inequality and (\ref{Empirical covariance useful
identity}) we get%
\begin{eqnarray*}
{\mathbb E}F &\leq& \left( {\mathbb E}\left\langle v,\sum_{i}\sigma _{i}x_{i}\right\rangle
^{2}\right) ^{1/2} \\
&=& \left( \sum_{i}\left\langle v,x_{i}\right\rangle
^{2}\right) ^{1/2}\leq \sqrt{ m\left\Vert \hat{\Sigma}\left( \mathbf{x}\right)
\right\Vert _{\infty } }.
\end{eqnarray*}%
(ii) Let $\mathbf{\sigma }$ be any configuration of the Rademacher
variables. For any $\sigma ^{\prime },\sigma ^{\prime \prime }\in \left\{
-1,1\right\} $ to replace $\sigma _{sj}$ we have 
\begin{equation*}
F\left( \mathbf{\sigma }_{\left( sj\right) \leftarrow \sigma ^{\prime
}}\right) -F\left( \mathbf{\sigma }_{\left( sj\right) \leftarrow \sigma
^{\prime \prime }}\right) \leq 2\left\vert \left\langle v,x_{j}\right\rangle
\right\vert ,
\end{equation*}%
so the conclusion follows from the bounded difference inequality, Theorem %
\ref{Theorem Concentration} (i).
\end{proof}

\begin{lemma}
\label{Lemma LTL aux 2}For $v_{1},\ldots,v_{K}\in H$ satisfying $\left\Vert
v_{k}\right\Vert \leq 1$, $\mathbf{x}\in H^{m}$ we have%
\begin{equation*}
{\mathbb E}\max_{k}\left\vert \left\langle v_{k},\sum_{i}\sigma _{i}x_{i}\right\rangle
\right\vert \leq \sqrt{2m\left\Vert \hat{\Sigma}\left( \mathbf{x}\right)
\right\Vert _{\infty }}\left( 2+\sqrt{\ln K}\right) .
\end{equation*}
\end{lemma}

\begin{proof}
Let $F_{k}=\left\vert \left\langle v_{k},\sum_{i}\sigma
_{i}x_{i}\right\rangle \right\vert $. Setting 
$c=\sqrt{m\left\Vert \hat{\Sigma}\left( \mathbf{x}\right)
\right\Vert _{\infty }}$
and using integration by parts we have for 
$\delta \geq 0$%
\par ${\mathbb E}\max_{k}F_{k} $
\begin{eqnarray*}
&\leq& c+\delta +\int_{\sqrt{m\left\Vert \hat{\Sigma}\left( 
\mathbf{x}\right) \right\Vert _{\infty }}+\delta }^{\infty }\max_{k}\Pr
\left\{ F_{k}\geq s\right\} ds \\
&\leq & c +\delta +\sum_{k}\int_{\delta }^{\infty }\Pr \left\{ F_{k}\geq
{\mathbb E}F_{k}+s\right\} ds \\
&\leq & c+\delta +\sum_{k}\int_{\delta }^{\infty }\exp \left( \frac{-s^{2}%
}{2m\left\Vert \hat{\Sigma}\left( \mathbf{x}\right) \right\Vert _{\infty }}%
\right) ds \\
&\leq & c+\delta +\frac{mK\left\Vert \hat{\Sigma}\left( \mathbf{x}\right)
\right\Vert _{\infty }}{\delta }\exp \left( \frac{-\delta ^{2}}{2m\left\Vert 
\hat{\Sigma}\left( \mathbf{x}\right) \right\Vert _{\infty }}\right) .
\end{eqnarray*}%
Above the first inequality is trivial, the second follows from Lemma \ref%
{Lemma LTL auc 1} (i) and a union bound, the third inequality follows from
Lemma \ref{Lemma LTL auc 1} (ii) and the last from a well known
approximation. The conclusion follows from substitution of $\delta =\sqrt{%
2m\left\Vert \hat{\Sigma}\left( \mathbf{x}\right) \right\Vert _{\infty }\ln
\left( eK\right) }$.
\end{proof}

\begin{proposition}
\label{Proposition uniform} Let $S_{\infty }\left( \mathcal{E}\right):={\mathbb E}_{\tau \sim \mathcal{E}}{\mathbb E}_{\left( \mathbf{x,y}\right) \sim
\mu _{\tau }^{m}}\left\Vert \hat{\Sigma}\left( \mathbf{x}\right) \right\Vert
_{\infty }$. With probability at least $1-\delta $ in the
multisample $\mathbf{Z}\sim \rho _{\mathcal{E}}^{T}$%
\begin{eqnarray}
&&\sup_{D\in \mathcal{D}_K}R_{\mathcal{E}}\left( A_{D}\right) -\frac{1}{T}%
\sum_{t=1}^{T}\hat{R}_{D}\left( \mathbf{z}_{t}\right)  \label{uniform body}
\\
&\leq& L\alpha K\sqrt{\frac{2\pi S_{1}\left( \mathbf{X}\right) }{T}} \notag  \\
&+& 4L\alpha \sqrt{\frac{S_{\infty }\left( \mathcal{E}\right) \left( 2+\ln
K\right) }{m}}+\sqrt{\frac{9\ln 2/\delta }{2T}}.\notag
\end{eqnarray}%
\end{proposition}

\begin{proof}
Following our strategy we write (abbreviating $\rho =\rho _{\mathcal{E}}$)%
\begin{eqnarray}
&&\sup_{D\in \mathcal{D}_K}R_{\mathcal{E}}\left( A_{D}\right) -\frac{1}{T}%
\sum_{t=1}^{T}\hat{R}_{D}\left( \mathbf{z}_{t}\right)
\label{Transfer bound decomposition} \notag \\
&\leq &\sup_{D\in \mathcal{D}_K}{\mathbb E}_{\tau \sim \mathcal{E}}{\mathbb E}_{\mathbf{z}\sim \mu
_{\tau }^{m}}\\
&& \left[ {\mathbb E}_{\left( x,y\right) \sim \mu _{\tau }}\left[ \ell
\left( \left\langle A_{D}\left( \mathbf{z}\right) ,x\right\rangle ,y\right)%
\right] -\hat{R}_{D}\left( \mathbf{z}\right) \right] \notag \\
 &+&\sup_{D\in \mathcal{D}_K}{\mathbb E}_{\mathbf{z}\sim \rho }\left[ \hat{R}_{D}\left( 
\mathbf{z}\right) \right] -\frac{1}{T}\sum_{t=1}^{T}\hat{R}_{D}\left( 
\mathbf{z}_{t}\right) 
 \notag
\end{eqnarray}%
and proceed by bounding each of the two terms in turn.

For any fixed dictionary $D$ and any measure $\mu $ on $\mathcal{Z}$ we have%
\begin{eqnarray*}
&&{\mathbb E}_{\mathbf{z}\sim \mu ^{m}}\left[ {\mathbb E}_{\left( x,y\right) \sim \mu }\left[
\ell \left( \left\langle A_{D}\left( \mathbf{z}\right) ,x\right\rangle
,y\right) \right] -\hat{R}_{D}\left( \mathbf{z}\right) \right]  \\
&\leq &{\mathbb E}_{\mathbf{z}\sim \mu ^{m}}\sup_{\gamma \in \mathcal{C}_\alpha}\bigg[
{\mathbb E}_{\left( x,y\right) \sim \mu }\left[ \ell \left( \left\langle D\gamma
,x\right\rangle ,y\right) \right] \\
&& -\frac{1}{m}\sum_{i=1}^{m}\ell \left(
\left\langle D\gamma ,x_{i}\right\rangle ,y_{i}\right) \bigg]  \\
&\leq &\frac{2}{m}{\mathbb E}_{\mathbf{z}\sim \mu ^{m}}{\mathbb E}_{\sigma }\sup_{\gamma \in 
\mathcal{C}_\alpha}\sum_{i=1}^{m}\sigma _{i}\ell \left( \left\langle D\gamma
,x_{i}\right\rangle ,y_{i}\right) \text{\small~~ [Theorem \ref{Theorem Rademacher
bound Expectation version}]} \\
&\leq &\frac{2L}{m}{\mathbb E}_{\mathbf{z}\sim \mu ^{m}}{\mathbb E}_{\sigma }\sup_{\gamma \in 
\mathcal{C}_\alpha}\sum_{k}\gamma _{k}\left\langle De_{k},\sum_{i=1}^{m}\sigma
_{i}x_{i}\right\rangle \text{\small~~[Lemma \ref{Lemma Rademacher Lipschitz}]} \\
&\leq &\frac{2L\alpha }{m}{\mathbb E}_{\mathbf{z}\sim \mu ^{m}}{\mathbb E}_{\sigma
}\max_{k}\left\vert \left\langle De_{k},\sum_{i=1}^{m}\sigma
_{i}x_{i}\right\rangle \right\vert \text{\small~~ [H\"{o}lder's ineq.]} \\
&\leq &\frac{2L\alpha }{m}{\mathbb E}_{\mathbf{z}\sim \mu ^{m}}\sqrt{2m\lambda _{\max
}\left( \hat{\Sigma}\left( \mathbf{x}\right) \right) }\left( 2+\sqrt{\ln K}%
\right) \text{~\small[Lemma \ref{Lemma LTL aux 2} (i)]} \\
&\leq &2L\alpha \sqrt{\frac{4{\mathbb E}_{\mathbf{z}\sim \mu ^{m}}\lambda _{\max
}\left( \hat{\Sigma}\left( \mathbf{x}\right) \right) \left( 2+\ln K\right) }{%
m}}\text{\small~~[Jensen's ineq.]}.
\end{eqnarray*}%
This gives the bound 
\par $\hspace{-.10truecm}{\mathbb E}_{\mathbf{z}\sim \mu ^{m}}{\hspace{-.08truecm}}\left[ {\mathbb E}_{\left( x,y\right) \sim \mu }\left[ \ell
\left( \left\langle A_{D}\left( \mathbf{z}\right) ,x\right\rangle ,y\right) %
\right] {\hspace{-.03truecm}}-{\hspace{-.02truecm}}\hat{R}_{D}\left( \mathbf{z}\right) \right]$
\begin{equation}
 \leq 4L\alpha \sqrt{%
\frac{{\mathbb E}_{\mathbf{z}\sim \mu ^{m}}\lambda _{\max }\left( \hat{\Sigma}\left( 
\mathbf{x}\right) \right) \left( 2+\ln K\right) }{m}}
\label{Bound on first term}
\end{equation}%
valid for every measure $\mu $ on $H\times 
%TCIMACRO{\U{211d} }%
%BeginExpansion
\mathbb{R}
%EndExpansion
$ and every $D\in \mathcal{D}_{K}$. Replacing $\mu $ by $\mu _{\tau }$,
taking the expectation as $\tau \sim \mathcal{E}$ and using Jensen's
inequality bounds the first term on the right hand side of (\ref{Transfer
bound decomposition}) by the second term on the right hand side of (\ref%
{uniform body}).

We proceed to bound the second term. From Corollary \ref{Corollary empirical
Rademacher bound} and Lemma \ref{Lemma Gauss dominates Rademacher} we get that
with probability at least $1-\delta $ in $\mathbf{Z}\sim \left( \rho _{%
\mathcal{E}}\right) ^{T}$%
\par $\sup_{D\in \mathcal{D}_K}{\mathbb E}_{\mathbf{z}\sim \rho }\left[ \hat{R}_{D}\left( 
\mathbf{z}\right) \right] -\frac{1}{T}\sum_{t=1}^{T}\hat{R}_{D}\left( 
\mathbf{z}_{t}\right)$
\begin{equation*}
\leq \frac{\sqrt{2\pi }}{T}{\mathbb E}_{\zeta }\sup_{D\in 
\mathcal{D}_K}\sum_{t=1}^{T}\zeta _{t}\hat{R}_{D}\left( \mathbf{z}_{t}\right) +%
\sqrt{\frac{9\ln 2/\delta }{2T}},
\end{equation*}%
where $\zeta _{t}$ is an orthogaussian sequence. Define two Gaussian
processes $\Omega $ and $\Xi $ indexed by $\mathcal{D}_K$ as 
\par
~~~~~~~~$\Omega _{D} =\sum_{t=1}^{T}\zeta _{t}\hat{R}_{D}\left( \mathbf{z}%
_{t}\right)$ 
\par and 
\par
~~~~~~~~$\Xi _{D} =\frac{L\alpha }{\sqrt{m}}\sum_{t=1}^{T}\sum_{i=1}^{m}%
\sum_{k=1}^{K}\zeta _{kij}\left\langle De_{k},x_{ti}\right\rangle$, 
\par where the $\zeta _{ijk}$ are also orthogaussian. Then for $D_{1},D_{2}\in 
\mathcal{D}_{K}$%

\begin{eqnarray*}
&&{\mathbb E}\left( \Omega _{D_{1}}-\Omega _{D_{2}}\right) ^{2}=\\
&=&\hspace{-.2truecm}\sum_{t=1}^{T}\left( 
\hat{R}_{D_{1}}\left( \mathbf{z}_{t}\right) -\hat{R}_{D_{2}}\left( \mathbf{z}%
_{t}\right) \right) ^{2} \\
&\leq &\hspace{-.2truecm}\sum_{t=1}^{T} \Bigg( \sup_{\gamma \in \mathcal{C}_\alpha}\frac{1}{m}%
\sum_{i=1}^{m}\ell \left( \left\langle D_{1}\gamma ,x_{ti}\right\rangle
,y_{ti}\right) \\
&&\hspace{-.2truecm} -\ell \left( \left\langle D_{2}\gamma ,x_{ti}\right\rangle
,y_{ti}\right) \Bigg) ^{2} \\
&\leq &\hspace{-.2truecm}L^{2}\sum_{t=1}^{T}\sup_{\gamma \in \mathcal{C}_\alpha}\left( \frac{1}{m}%
\sum_{i=1}^{m}\left\langle \gamma ,\left( D_{1}^{\top }-D_{2}^{\top }\right)
x_{ti}\right\rangle \right) ^{2}\text{\small Lipschitz} \\
&\leq &\hspace{-.2truecm}\frac{L^{2}}{m}\sum_{t=1}^{T}\sup_{\gamma \in \mathcal{C}_\alpha%
}\sum_{i=1}^{m}\left\langle \gamma ,\left( D_{1}^{\top }-D_{2}^{\top
}\right) x_{ti}\right\rangle ^{2}\text{\small ~~Jensen} \\
&\leq &\hspace{-.2truecm}\frac{L^{2}\alpha ^{2}}{m}\sum_{t=1}^{T}\sum_{i=1}^{m}%
\left\Vert \left( D_{1}^{\top }-D_{2}^{\top }\right) x_{ti}\right\Vert ^{2}%
\text{\small(Cauchy-Schwarz)} \\
&=&\hspace{-.2truecm}\frac{L^{2}\alpha ^{2}}{m}\sum_{t=1}^{T}\sum_{i=1}^{m}\sum_{k=1}^{K}%
\left( \left\langle D_{1}e_{k},x_{ti}\right\rangle -\left\langle
D_{2}e_{k},x_{ti}\right\rangle \right) ^{2} \\
&=&\hspace{-.1truecm} {\mathbb E}\left( \Xi _{D_{1}}-\Xi _{D_{2}}\right) ^{2}.
\end{eqnarray*}%
So by Slepian's Lemma%
\par ${\mathbb E}\sup_{D\in \mathcal{D}_K}\sum_{t=1}^{T}\zeta _{j}\hat{R}_{D}\left( \mathbf{z}%
_{t}\right) $
\begin{eqnarray*}
&=&{\mathbb E}\sup_{D\in \mathcal{D}_K}\Omega _{D}\leq {\mathbb E}\sup_{D\in \mathcal{%
D}}\Xi _{D} \\
&=&\frac{L\alpha }{\sqrt{m}}{\mathbb E}\sup_{D\in \mathcal{D}_K%
}\sum_{t=1}^{T}\sum_{i=1}^{m}\sum_{k=1}^{K}\zeta _{kij}\left\langle
De_{k},x_{ti}\right\rangle  \\
&=&\frac{L\alpha }{\sqrt{m}}{\mathbb E}\sup_{D\in \mathcal{D}_K}\sum_{k=1}^{K}\left%
\langle De_{k},\sum_{t=1}^{T}\sum_{i=1}^{m}\zeta _{kij}x_{ti}\right\rangle 
\\
&\leq &\frac{L\alpha }{\sqrt{m}}\sup_{D\in \mathcal{D}_K}\left(
\sum_{k}\left\Vert De_{k}\right\Vert ^{2}\right) ^{1/2} \\
&& {\mathbb E}\left(
\sum_{k}\left\Vert \sum_{t,i}\zeta _{tki}x_{ti}\right\Vert ^{2}\right) ^{1/2}
\\
&\leq &\frac{L\alpha \sqrt{K}}{\sqrt{m}}\left( \sum_{k}{\mathbb E}\left\Vert
\sum_{t,i}\zeta _{tki}x_{ti}\right\Vert ^{2}\right) ^{1/2} \\
&\leq &\frac{L\alpha \sqrt{K}}{\sqrt{m}}\left( \sum_{k}\sum_{t,i}\left\Vert
x_{ti}\right\Vert ^{2}\right) ^{1/2} \hspace{-.15truecm} \leq L\alpha K\sqrt{TS_{1}\left( \mathbf{X}\right) }.
\end{eqnarray*}%
We therefore have that with probability at least $1-\delta $ in the draw of
the multi sample $\mathbf{Z\sim }\rho ^{T}$ 
\par $\sup_{D\in \mathcal{D}_K}{\mathbb E}_{\mathbf{z}\sim \rho }\left[ \hat{R}_{D}\left( 
\mathbf{z}\right) \right] -\frac{1}{T}\sum_{i=1}^{T}\hat{R}_{D}\left( 
\mathbf{Z}_{t}\right)$
\begin{equation}
 \leq L\alpha K\sqrt{\frac{2\pi S_{1}\left( \mathbf{%
X}\right) }{T}}+\sqrt{\frac{9\ln 2/\delta }{2T}}.
\label{Expected empirical error bound}
\end{equation}%
which in (\ref{Transfer bound decomposition}) combines with (\ref{Bound on first term}) to give the conclusion.
\end{proof}

\begin{proof}[Proof of Theorem \protect\ref{Theorem Main}]
Let $D_{\rm opt}$ and $\gamma _{\tau }$ the minimizers in the definition of $%
R_{\rm opt}$, so that%
\begin{equation*}
R_{\rm opt}={\mathbb E}_{\tau \sim \mathcal{E}}{\mathbb E}_{\left( x,y\right) \sim \mu _{\tau }}\ell %
\left[ \left( \left\langle D_{\rm opt}\gamma _{\tau },x\right\rangle ,y\right) %
\right] .
\end{equation*}%
$R_{\mathcal{E}}\left( A_{D\left( \mathbf{Z}\right) }\right) -R_{\rm opt}$ can be decomposed as the sum of four terms,
\begin{eqnarray}
%&& R_{\mathcal{E}}\left( A_{D\left( \mathbf{Z}\right) }\right) -R_{\rm opt} \\
&&\left( R_{\mathcal{E}}\left( A_{D\left( \mathbf{Z}\right) }\right) -\frac{%
1}{T}\sum_{t=1}^{T}\hat{R}_{D\left( \mathbf{Z}\right) }\left( \mathbf{z}%
_{t}\right) \right)  \label{First} \\
&&+\left( \frac{1}{T}\sum_{t=1}^{T}\hat{R}_{D\left( \mathbf{Z}\right)
}\left( \mathbf{z}_{t}\right) -\frac{1}{T}\sum_{t=1}^{T}\hat{R}%
_{D_{\rm opt}}\left( \mathbf{z}_{t}\right) \right)  \label{Second term} \\
&&+\frac{1}{T}\sum_{t=1}^{T}\hat{R}_{D_{\rm opt}}\left( \mathbf{z}_{t}\right)
-{\mathbb E}_{\mathbf{z\sim \rho }}\hat{R}_{D_{\rm opt}}\left( \mathbf{z}\right)
\label{third term} \\
&&+{\mathbb E}_{\tau \sim \mathcal{E}}\bigg[ {\mathbb E}_{\mathbf{z}\sim \mu _{\tau }^{m}}\hat{R}%
_{D_{\rm opt}}\left( \mathbf{z}\right) \notag \\
&& -{\mathbb E}_{\left( x,y\right) \sim \mu _{\tau
}}\left[ \ell \left( \left\langle D_{\rm opt}\gamma _{\tau },x\right\rangle
,y\right) \right] \bigg] .  \label{fourth term}
\end{eqnarray}%
By definition of $\hat{R}$ we have for every $\tau $ that 
\par ${\mathbb E}_{\mathbf{z}\sim \mu _{\tau }^{m}}\hat{R}_{D_{\rm opt}}\left( \mathbf{z}\right)$
\begin{eqnarray*}
&=&{\mathbb E}_{\mathbf{z}\sim \mu _{\tau }^{m}}\min_{\gamma \in \mathcal{C}_\alpha}\frac{1}{m%
}\sum_{i=1}^{m}\ell \left[ \left( \left\langle D_{\rm opt}\gamma
,x_{i}\right\rangle ,y_{i}\right) \right] \\
&\leq &{\mathbb E}_{\mathbf{z}\sim \mu _{\tau }^{m}}\frac{1}{m}\sum_{i=1}^{m}\ell %
\left[ \left( \left\langle D_{\rm opt}\gamma _{\tau },x_{i}\right\rangle
,y_{i}\right) \right] \\
&=& {\mathbb E}_{\left( x,y\right) \sim \mu _{\tau }}\ell \left[ \left( \left\langle
D_{\rm opt}\gamma _{\tau },x\right\rangle ,y\right) \right] .
\end{eqnarray*}%
The term (\ref{fourth term}) above is therefore non-positive. By Hoeffding's
inequality the term (\ref{third term}) is less than $\sqrt{\ln \left(
2/\delta \right) /2T}$ with probability at least $1-\delta /2$. The term \eqref%
{Second term} is non-positive by the definition of $D\left( \mathbf{Z}%
\right) $. Finally we use Proposition \ref{Proposition uniform} to obtain
with probability at least $1-\delta /2$ that%
\par $R_{\mathcal{E}}\left( A_{D\left( \mathbf{Z}\right) }\right) -\frac{1}{T}%
\sum_{t=1}^{T}\hat{R}_{D\left( \mathbf{Z}\right) }\left( \mathbf{z}%
_{t}\right) $
\begin{eqnarray*}
&\leq &\sup_{D\in \mathcal{D}_K}R_{\mathcal{E}}\left( A_{D}\right)
-\frac{1}{T}\sum_{t=1}^{T}\hat{R}_{D}\left( \mathbf{z}_{t}\right) \\
&\leq &L\alpha K\sqrt{\frac{2\pi S_{1}\left( \mathbf{X}\right) }{T}} \\
&+& 4L\alpha \sqrt{\frac{S_{\infty }\left( \mathcal{E}\right) \left( 2+\ln
K\right) }{m}}+\sqrt{\frac{9\ln 4/\delta }{2T}}.
\end{eqnarray*}%
Combining these estimates on (\ref{First}), (\ref{Second term}), (\ref{third
term}) and (\ref{fourth term}) in a union bound gives the conclusion.
\end{proof}

\end{document}

% This document was modified from the file originally made available by
% Pat Langley and Andrea Danyluk for ICML-2K. This version was
% created by Lise Getoor and Tobias Scheffer, it was slightly modified  
% from the 2010 version by Thorsten Joachims & Johannes Fuernkranz, 
% slightly modified from the 2009 version by Kiri Wagstaff and 
% Sam Roweis's 2008 version, which is slightly modified from 
% Prasad Tadepalli's 2007 version which is a lightly 
% changed version of the previous year's version by Andrew Moore, 
% which was in turn edited from those of Kristian Kersting and 
% Codrina Lauth. Alex Smola contributed to the algorithmic style files.  